\documentclass{article}

\usepackage[preprint,nonatbib]{neurips_2023}

\usepackage{graphicx}
\usepackage{xcolor}

\usepackage{subfig}

\usepackage[utf8]{inputenc}
\usepackage[numbers,sort&compress]{natbib}

\usepackage{url}   

\usepackage{amsmath, latexsym}
\usepackage{mathtools}
\usepackage{dsfont}
\usepackage{amsthm}
\usepackage{amsfonts}

\title{Corridor Geometry in Gradient-Based Optimization}

\author{
Benoit Dherin$^*$ \\
Google \\
dherin@google.com
\and
{\bf Mihaela Rosca}$^*$ \\
Google DeepMind, London \\
mihaelacr@google.com
}

\newtheorem{theorem}{Theorem}[section]

\newtheorem{lemma}[theorem]{Lemma}
\newtheorem{definition}[theorem]{Definition}

\newtheorem{remark}[theorem]{Remark}
\newtheorem{example}[theorem]{Example}
\newtheorem{proposition}[theorem]{Proposition}
\newcommand{\R}[0] {\mathbb R}

\begin{document}

\maketitle
\def\thefootnote{*}\footnotetext{Equal contribution}\def\thefootnote{\arabic{footnote}}

\begin{abstract}%
We characterize regions of a loss surface as corridors when the continuous curves of steepest descent---the solutions of the gradient flow---become straight lines. We show that corridors provide insights into gradient-based optimization, since corridors are exactly the regions where gradient descent and the gradient flow follow the same trajectory, {\color{black} while the loss decreases linearly}.
As a result, inside corridors there are no implicit regularization effects or training instabilities that have been shown to occur due to the drift between gradient descent and the gradient flow. {\color{black}
Using the loss linear decrease on corridors, we devise a learning rate adaptation scheme for gradient descent; we call this scheme Corridor Learning Rate (CLR). The CLR formulation coincides with a special case of Polyak step-size, discovered in the context of convex optimization. The Polyak step-size has been  shown recently to have also good convergence properties for neural networks; we further confirm this here with results on CIFAR-10 and ImageNet.
}
\end{abstract}

\section{Introduction}

In recent years, a fruitful approach to understanding optimization in deep learning has come from studying the discrepancy between gradient descent (GD)---at the basis for most modern deep learning optimizers---and its continuous-time counterpart, the gradient flow (GF). On one hand, the  discrepancy between GD and the GF has been used to find  sources of implicit regularization and training instabilities~\citep{gunasekar2017implicit,soudry2018implicit,barrett2021implicit,smith2021on,ghosh2023implicit,symmetry,miyagawatoward,rosca2023implicit,rosca2021discretization,rosca2023on} and with it novel insights into the effect of learning rates, batch sizes, and neural landscapes. 
On the other hand, it has been empirically observed that there are regions in neural network training where GD closely follows the GF \citep{cohen2021gradient,rosca2023on,elkabetz2021continuous_vs_discrete}, with a theoretical bound between the trajectory of GD and the GF constructed by \citet{elkabetz2021continuous_vs_discrete}.
Taken together, these results showcase the importance of understanding the conditions of the loss landscape under which GD and the GF follow the same trajectory, which can unlock \textit{when} implicit regularization and training instabilities are expected to occur. To this end, \citet{rosca2023on} find a sufficient condition based on the loss gradient and Hessian for when GD follows the GF exactly.

Our main contribution is to add to this body of work, by
\textit{characterizing the local geometry of regions where the trajectories of GD coincide with those of the GF for general loss surfaces}.  Geometrically, we find that these regions exactly coincide with regions where the solutions of the GF become straight lines, regions we call \textbf{corridors}. We show that a necessary and sufficient condition for a region to be a corridor is that the image $Hg$ of loss gradient $g$ by the loss Hessian $H$ vanishes.  
Although exact corridors (i.e. where $Hg$ is exactly zero) may not exist on neural loss landscapes, we believe that understanding corridors can provide an insight into optimization and implicit regularization in deep learning, especially in regions where $Hg$ is small (as opposed to zero). 
Based on the assumption that the loss landscape is a corridor, 
we devise a learning rate adaptation scheme for GD,  which we call the Corridor Learning Rate (CLR); inside corridors, the CLR converges to the optimum in one step. 
Outside exact corridors, we find empirically that 
 CLR converges on multiple datasets (CIFAR-10~\citep{cifar10} and Imagenet~\citep{deng2009imagenet}) and multiple neural network architectures. {\color{black} Perhaps surprisingly, the CLR coincides with a special case of Polyak step-size for GD, introduced by \citet{polyak1987introduction} as optimal in the non-smooth convex case. The Polyak step-size has been shown to also be optimal in other convex settings (e.g. smooth and strongly convex) by \citet{hazan2019revisiting}. Recently, a stochastic variant of the Polyak step-size has been used in the context of neural optimization where, under mild assumptions, theoretical convergence guaranties have been obtained by \citet{loizou2021stochastic}; the authors observe fast convergence 
on MNIST~\citep{lecun1989backpropagation}, CIFAR-10~\citep{cifar10}, CIFAR-100~\citep{cifar10}, for the ResNet \cite{resnets} and DenseNet \cite{huang2017densenet} architectures as well as for matrix factorization.
Our experiments support their evidence on efficiency of this learning rate formulation in neural network training.
We  note that the CLR is not the focus of this work, nor do we aim to make it competitive with existing approaches as in  \citet{orvieto2022dynamics}: our aim is to explore the notion of corridor and demonstrate its usefulness  in the context of deep learning}. 

\section{Corridors}

In this section, we define the geometrical notion of corridor, provide an equivalent analytic criterion for corridors, and show that corridors coincide with regions where GD and the GF follow the same trajectory. 
Our results are valid for twice differentiable loss functions $E:\R^n\rightarrow \R$. We denote by $g(\theta) = \nabla E(\theta)$ the loss gradient and its Hessian by $H(\theta) = \nabla\nabla E(\theta)$. In deep learning $\theta$ denotes the parameter of a neural network, and $E(\theta)$ denotes a measure of discrepancy between neural network predictions and the ground truth averaged over all the examples in the training set. With our notation, the GF differential equation becomes
$
    \dot \theta = -g(\theta).
$
Its solutions $\theta(t)$ are given by the unique curves of continuous steepest descent.

\vspace{0.2cm}

\begin{definition}
We define a {\bf corridor} as a domain $U\subset \R^n$ of a loss surface with loss $E:\R^n\rightarrow \R$ iff the solutions of the GF (i.e., the curves $\theta(t)$ satisfying $\dot \theta(t) = -g(\theta(t))$ are straight lines  {\color{black} traveled at constant speed on $U$, (i.e.,   $\theta(t) = \theta(0) + t v$ for a velocity vector $v$)}. Inside corridors, we thus refer to the GF solutions as the {\bf lines of steepest descent}.
\label{def:corridor}
\end{definition}

The main theorem in this section states that a subset $U$ is a corridors  if and only if ${H(\theta) g(\theta) = 0, \forall \theta \in U}$ (Theorem~\ref{thm:corridor}). We first start with a few corridor examples, to provide an intuition about what they are and we motivate our main theorem by observing that $Hg=0$ for all these examples.

{\color{black}

\subsection{Examples of corridors}

{\color{black}
\begin{example} (Linear pieces.)
In dimension one, only linear pieces are corridors. To see why, consider the loss $E:\R\rightarrow \R$, and $\theta(t)$ the solution of the GF i.e., $\dot \theta = - \nabla E(\theta)$. Consider the time interval $[0, T]$ under which $E$ is a corridor; we show that $\nabla E(\theta)$ is constant in that interval. 
By the corridor assumption, we have that $\theta(t) = \theta(0) + t v,$  satisfies the GF equation for a certain $v$, which implies that $\nabla E(\theta(t))$ is constant (equals to $v$) for all $t \in [0, T]$, 
and thus $E$ is linear in $[\theta(0), \theta(T)]$. Since for a linear function, its Hessian vanishes, we have that $Hg = 0$ on the corridor.

\end{example}
}

\vspace{0.2cm}

\begin{example}(Convex n-dimensional cones.)
The loss $E: \R^d \rightarrow \R$,  $E(\theta) = \alpha \|\theta\|$ (where $\|\cdot\|$ is the $l_2$ norm) has no linear piece, but it is a corridor. This corridor can be described as the collection of half-lines passing through a point in the circle ${C = \{(\theta,1):\: \|\theta\|^2 = 1/\alpha^2\}\subset \R_{\theta}\times\R}$ and stopping at the origin $(0,0)$. The projection of these lines on the parameter plane $\R_{\theta}$ are the lines of steepest descent and the solutions of the GF.
The $n$-dimensional cone is a convex surface with a single global minima.
Note that the gradient of $E(\theta) = \alpha \|\theta\|$ is $g = \alpha \frac{\theta}{\|\theta\|}$, and its Hessian is $H = \alpha(\frac{1}{\|\theta\|} - 
\frac{\theta\theta^T}{\|\theta\|^3})$ and thus $Hg = 0$, except at the origin.
\end{example}

\vspace{0.2cm}

\begin{example}(Concave n-dimensional Cones.)
Turning the previous example on its head, we obtain the concave n-dimensional cone with loss function $E(\theta) = \beta - \alpha \|\theta\|_2$. A similar computation shows that $Hg = 0$, except at the origin.
\label{ex:inverted_cone}
\end{example}

\vspace{0.2cm}

\begin{example} (Ruled-surfaces.)
Both the 2-dimensional convex and concave cones are examples of \emph{ruled-surfaces} \citep{wiki_ruled_surface}, which are surfaces defined by the property that there is always a line lying on the surface through each of its points as in Figure \ref{fig:individual_corridors}. As Proposition \ref{prop:linear_loss} below shows, in $\R^2$, if a region is a corridor then it is ruled-surface (the converse being not necessarily true) allowing for much more complicated shapes than just linear pieces. %
\end{example}

\vspace{0.2cm}

\begin{example}(Multiple corridors).
The previous examples were examples of loss surfaces entirely formed by a single corridor. We can glue corridors together to obtain fairly complicated loss surfaces with a hierarchical structure.
Figure \ref{fig:multiple_corridors} displays such an example, where wavy convex and concave cones are glued to pieces of ruled-surfaces and linear pieces. Note that $Hg = 0$ holds almost everywhere except on a set measure null, corresponding to the gluing lines, where the loss is possibly not differentiable. This example is non-convex (and non-smooth at the gluing lines). 
\end{example}

\vspace{0.2cm}

\begin{example}(Quadratic Counter-example.)
Quadratic losses do not possess corridors. 
For instance, the solutions of  the GF for the loss $E(\theta) = \theta^2/2$ are lines  toward the origin $\theta(t) = \theta e^{-t}$, but are travelled at non-constant speed $\dot \theta (t) = \theta e^{-t}$. Note that $H g = \theta$ does not vanish, except a $0$. See Appendix \ref{appendix:quadratic} for a proof that quadratic functions in higher dimensions do not possess corridors. 
\label{ex:quadratic}
\end{example}

}

\begin{figure}[t]
\centering
\begin{subfloat}[Individual corridors.]{
\includegraphics[width=0.2\textwidth]{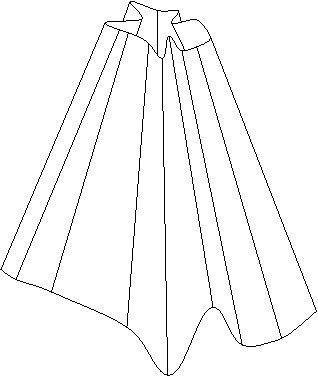}%
\includegraphics[width=0.15\textwidth]{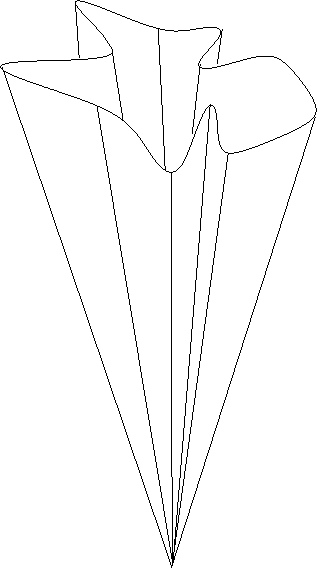}%
\label{fig:individual_corridors}
}
\end{subfloat}
\begin{subfloat}[Loss surface formed from multiple corridors.]{
\includegraphics[width=0.45\textwidth]{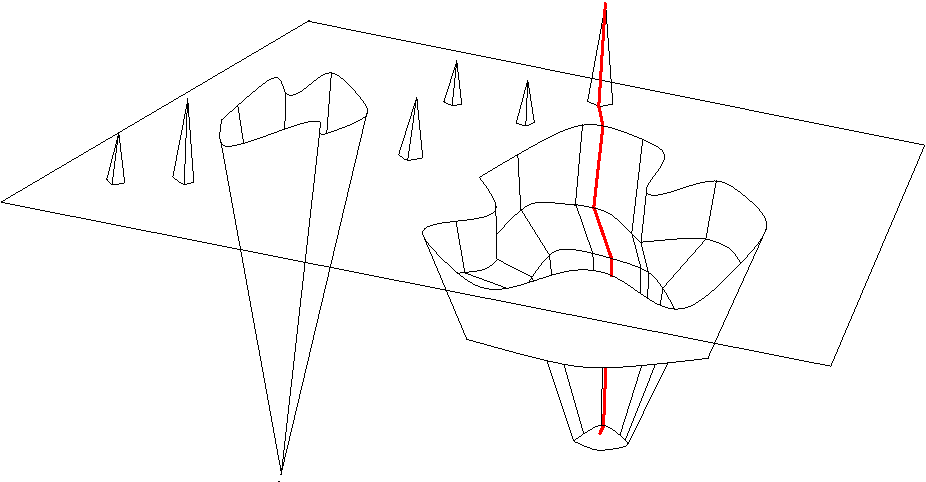}%
\label{fig:multiple_corridors}
}
\end{subfloat}
\caption{
Visualizing corridors. 
\protect\subref{fig:individual_corridors}: 
Ruled surfaces forming a corridor, with example lines of steepest descent.
\protect\subref{fig:multiple_corridors}: 
A line of steepest descent (shown in red) formed on a loss surface constructed from multiple corridors. 
}
\label{fig:ruled_surfaces}
\end{figure}

\subsection{Characterizing corridors: a necessary and sufficient condition}

Theorem~\ref{thm:corridor} below gives an analytical characterization of corridors in terms of the Hessian. Before stating it, let us start with two lemmas:

\begin{lemma}\label{lemma:hessian_ode}
Suppose that $\theta(t)$ is a solution of $\dot \theta = - g(\theta)$. Then $-Hg$ measures the rate of change of the loss gradient under the GF: 
\begin{equation}\label{eq:hessian_ode}
    \frac{d}{dt}g(\theta(t)) = - H(\theta(t)) g(\theta(t)).
\end{equation}
\end{lemma}

The proof follows immediately from the chain rule
$$\frac{d}{dt}g(\theta(t)) = \nabla_\theta g(\theta(t)) \dot \theta(t) = - H(\theta(t))g(\theta(t)).$$

\begin{lemma}
Inside corridors, GF solutions are of the form $\theta(t) = \theta -tg(\theta)$ and the loss gradient is constant on this line: $g(\theta - tg(\theta)) = g(\theta)$ for all $t$ as long as the line remains a solution.
\label{lemma:cost_grad}
\end{lemma}
\begin{proof}
By Def.~\ref{def:corridor},
 solutions to the GF inside a corridor are of the form  $\theta(t) = a + tb$.  Setting $t=0$ implies that $a=\theta(0)$. We also have that inside a corridor $\dot \theta(t) = b$, and thus $b =  - g(\theta(t)) = -g(\theta(0) + tb)$. So $g$ is constant on $\theta(t)$ and we obtain that $b = -g(\theta(0))$ by setting $t=0$.
\end{proof}

{\color{black}
The next proposition shows that the loss decreases linearly on a line of steepest descent within a corridor; moreover, the trajectories followed on the loss surface when the parameter follows a line of steepest descent is itself a line:
}
{\color{black}
\begin{proposition} \label{prop:linear_loss} On a corridor $U\subset \R^n$ the loss decreases linearly under the GF. More precisely, on the line of steepest descent through $\theta$, i.e., $\theta(t) = \theta - tg(\theta)$, we have that ${E(\theta(t)) = E(\theta) - t\| g(\theta)\|^2}$. Moreover, the line of steepest descent lifts to a line through $\theta$ on the loss surface: ${t\mapsto \left(\theta - g(\theta)t,\,  E(\theta) - t\| g(\theta)\|^2\right)}$.
\end{proposition}
\begin{proof}
By the chain rule we have that the loss under the GF satisfies the following differential equation $\frac{d}{dt} E(\theta(t)) = \nabla_\theta E(\theta) \dot \theta(t) = - \|g(\theta(t))\|^2$. Since inside a corridor the gradient is constant along the solution of the GF and equal to $g(\theta)$, we can thus solve this equation and we obtain that $E(\theta(t)) = E(\theta) - t \|g(\theta)\|^2$. The last part of the proposition then follows by the definition of the loss surface as the set of points $\{(\theta,\, E(\theta)):\: \theta\in \R^n\}$. 
\end{proof}
}
We now give an analytical characterisation of corridors in terms of the Hessian. 
\begin{theorem}\label{thm:corridor}
Let $E:\R^n\rightarrow \R$ be a twice differentiable loss function and $U\subset \R^n$ be a domain of the parameter space. Then $U$ is a corridor if and only if $H(\theta)g(\theta) = 0$ for all $\theta \in U$.
\end{theorem}

\begin{proof}
In a corridor $U$, $g$ is constant on the lines $\theta(t)$ of steepest descent by Lemma \ref{lemma:cost_grad}, and thus $\frac {d}{dt}g(\theta(t)) =  - H(\theta(t)) g(\theta(t)) = 0$ by  Lemma~\ref{lemma:hessian_ode}. Since this holds for all $\theta(0) \in U$, this implies $H(\theta) g(\theta) = 0$ on $U$. Conversely, if $H(\theta) g(\theta) = 0, \forall \theta \in U$, we have that $$\frac{d}{dt}g(\theta(t)) = - H(\theta(t)) g(\theta(t)) = 0,$$ which implies that $g(\theta(t)) = g(\theta(0))$ and the GF ODE on corridors becomes $\dot \theta = -g(\theta(0))$. This last equation is easy to solve: $$\theta(t) = \theta(0) + \int_0^t - g(\theta(t')) dt' = \theta(0) - t g(\theta_0),$$ entailing that the GF solution is a straight line and we are thus in a corridor.
\end{proof}

\section{Corridors are important for optimization} \label{section:corridors_and_optimization}
We now show that corridors are the only regions where GD steps on the GF. In
Appendix~\ref{app:beyond_gd}, we show that inside corridors Runge-Kutta4~\cite{hairer2006geometric} and momentum~\cite{polyak1964some} also follow the GF trajectory.
\begin{theorem}\label{thm:iterates_and_corridors}
Let $E:\R^n\rightarrow \R$ be a twice differentiable loss function and $U\subset \R^n$ be a domain of the parameter space. Then $U$ is a corridor if and only if the GD iterates $\theta_n$ step on the solution $\theta(t)$ of the GF: $\theta(nh) = \theta_n$. 
\end{theorem}
\begin{proof}
In a corridor, we prove by induction that $\theta_n = \theta(nh)$, where $\theta(t) = \theta_0 - t g(\theta_0)$ is the line of steepest descent through $\theta_0$ from Lemma \ref{lemma:cost_grad}. We obtain the base case by definition of GD: $\theta_1 = \theta_0 - h g(\theta_0) = \theta(h)$. Now suppose by induction that we have $\theta_n = \theta(nh)$. This implies that
$
\theta_{n+1} = \theta(nh) - hg(\theta(nh))
= \theta_0 - nh g(\theta_0) - h g(\theta_0) = \theta((n+1)h)
$, where we used that  $g(\theta_0 - nhg(\theta_0)) = g(\theta_0)$ according to Lemma \ref{lemma:cost_grad}.

Conversely, we assume that $\theta_n = \theta(nh)$ for all initial conditions $\theta \in U$ and all learning rates $h$ as long as we remain on $U$. By setting $n=1$, we obtain that the solution of the GF through $\theta$ becomes $\theta(h) = \theta_1 = \theta - hg(\theta)$. The Taylor expansion of the GF solution with initial condition $\theta$ is $$\theta(h) = 
\theta - h \theta' + \frac{h^2}{2} \theta'' + \mathcal{O}(h^3) = \theta - h g(\theta) - \frac{h^2}{2} \frac{d}{dt}  g(\theta) + \mathcal{O}(h^3)= \theta - h g(\theta) + \frac{h^2}{2} H(\theta) g (\theta)+ \mathcal{O}(h^3),$$ where the first equality comes from the definition of the GF, and the second one from Lemma~\ref{lemma:hessian_ode}.
From here, we have that $H(\theta) g(\theta) = 0$, and thus we are inside a corridor by Theorem \ref{thm:corridor}.
\end{proof}

Theorem \ref{thm:iterates_and_corridors} shows that GD and the GF follow the same trajectory if and only if the landscape is a corridor. This entails that as long as a corridor is not exited, GD exactly minimizes the loss function $E$, {\color{black} which decreases linearly (Proposition \ref{prop:linear_loss})}, and implicit regularization effects found due to the discrepancy between GD and the GF such as Implicit Gradient Regularization ~\citep{barrett2021implicit} have no effect.

While inside corridors GD follows the GF exactly, this is not the case for Stochastic Gradient Descent (SGD),  where a different data batch $B$ is used at each iteration. The batch gradient $g_B$ can be rewritten as  $g_B = g + \epsilon_B$ where $g$ denotes the full batch gradient, and the residual $\epsilon_B$ is due to the noise induced by the sampling of batch $B$. Despite being inside a corridor, SGD will not follow the GD trajectory exactly, since $\theta(h) = \theta - h g(\theta) \ne \theta - hg_B(\theta) = \theta - h g(\theta) -h \epsilon_B $.
Thus, as the batch size decreases and the learning rate increases, implicit regularization will occur for SGD inside corridors. In Appendix~\ref{section:sgd}, we quantify this further based on results on implicit regularization for SGD from ~\citet{smith2021on}: we show that while the gradient norm implicit regularizer vanishes inside corridors, the per batch consistency implicit regularizer remains.

We show that  Runge-Kutta4 as well as GD with momentum also slide on the lines of steepest descent (i.e. the solutions of the GF) in a corridor in the Appendix. For  Runge-Kutta4, the discrete steps exactly follow GF solutions inside a corridor (Appendix \ref{section:rk4}). For momentum, we show in Appendix~\ref{section:momentum} that if the momentum iterates enter a corridor with velocity zero or one aligned with the initial corridor gradient, momentum will follow the lines of steepest descent, but at a different speed than the GF. For momentum with zero initial velocity, we recover the ODE describing momentum found by~\citet{ghosh2023implicit} and~\citet{cattaneo2023implicit}, but the implicit regularization term vanishes, like for GD. Outside of these conditions, however, momentum may drift from the GF, creating theoretical conditions for implicit regularization.
 These results suggests the importance of understanding corridors for optimization; we turn our attention to this in the next section.

\section{Corridor learning rate adaptation}

We devise a simple learning rate adaptation for GD based on the idea of  corridors.
{\color{black}
We assume that a point $\theta$ is in a corridor of the loss surface; as shown in Theorem \ref{thm:iterates_and_corridors}, the gradient iterates starting at $\theta$ coincide with the line of steepest descent $\theta(h) = \theta - h g(\theta)$ on the corridor. Moreover, by Proposition \ref{prop:linear_loss}, the loss on this steepest descent line decreases linearly as $E(\theta(h)) = E(\theta) - h\|g(\theta)\|^2$.
}
Assuming the minimal value of the loss function is $E(\theta^*)$ is 0, which holds for common losses used in deep learning with models with sufficient capacity, {\color{black} } and solving for $h$ in order to obtain the step size that allows us to reach that optimal value in \textit{one step} whilst in a corridor, we obtain:
\begin{align}
    E(\theta)  - h \|g(\theta)\|^2 = 0,
\end{align}
which if solved for $h$ gives the {\bf Corridor Learning Rate (CLR)} and corresponding gradient update:
\begin{equation}
    h(\theta) = \frac{E(\theta)}{\|g (\theta)\|^2},
    \quad\quad
    \theta' = \theta - \frac{E(\theta)}{\|g(\theta)\|^2}  g(\theta),
\end{equation}

{\color{black}
\subsection{Connection to Polyak step size}

{\color{black}
}
The CLR is connected to the Polyak step size~\citep{polyak1987introduction} for GD, where the learning rate is set to $h = \frac{E(\theta) - E(\theta^\textbf{}*)}{\|g (\theta)\|^2}$ and $E(\theta^*)$ is the loss function at the optimum. Thus, if $E(\theta^*) =0$ (as is the case for over-parameterized neural networks), the CLR recovers the Polyak step size. Polyak~\cite{polyak1987introduction} argues that this step size is optimal when $E$ is non-smooth and convex. This optimality result has been extended to  various additional convex settings by \citet{hazan2019revisiting} in particular to  $\beta$-smooth and strongly convex functions. They further provide an algorithm which does not require knowledge of $E(\theta^*)$, but only a lower bound on $E(\theta^*)$; assuming the lower bound is $0$, the learning rate used is $h = \frac{E(\theta)}{2 \|g (\theta)\|^2}$, differing from the CLT by a multiplier of $1/2$. 

Beyond the convex case, \citet{loizou2021stochastic} adapted the Polyak step-size to the stochastic setting and analysed convergence properties in a number of smooth cases, and in particular: 1) The Polyak-Lojasiewicz (PL) condition, which in the over parameterized regime assumes that the adaptive step-size $E(\theta) / \|\nabla E(\theta)\|^2$ is bounded, and 2) a condition on the gradient $\mathbb E(\|\nabla E_B(\theta)\|^2) \leq \rho \|\nabla E(\theta)\|^2 + \delta$ for some $\rho, \delta \geq 0$ and where $E_B$ is the batch loss; the latter condition has been previously studied by \citet{bottou2018optimization}. \citet{loizou2021stochastic} find theoretical convergence guarantees in both cases: namely to an exact  minimizer in case 1) and  to a neighborhood of the global minima in case 2), provided some modifications such as clipping the step-size for too large values and multiplying it by a constant factor as an additional hyper-parameter.   \citet{orvieto2022dynamics} considered further variations guarantying convergence to an exact minimizer in the same setting.

\textbf{The CLT and Polyak assumptions:}  The CLT is derived by assuming the loss landscape is a corridor, and then solving the learning rate that reaches the optimum loss value of $0$ in \textit{one step}. In contrast, the initial formulation of the Polyak learning rate is derived assuming a convex loss function. We note that the set of convex functions and corridors have non null differences; Example~\ref{ex:quadratic} has shown that quadratic functions do not allow corridors, while Example~\ref{ex:inverted_cone} has shown a concave function that is a corridor. 
The more recent use of the Polyak learning rate, such as that of ~\citet{loizou2021stochastic} in the non-convex case, employs other assumptions, such as the PL condition. The PL condition is a broad condition, which holds for large parts of neural network landscapes; this can be further observed by the fact that the PL condition is equivalent to assuming that the Polyak step-size is finite, and the Polyak step-size has been used successfully to train neural networks. While corridors are not likely to cover entire neural network landscapes due edge of stability results \cite{cohen2021gradient,ma2022quadratic}, it has been observed that GF follows GD early in neural network training  \cite{cohen2021gradient,rosca2023on}.
Furthermore, the understanding of corridors in neural networks is relevant from an implicit regularisation point of view, as implicit regularisation derived from the discretization error of gradient descent will not be present in corridors. 
}
\subsection{Empirical evidence from deep learning}

While the landscape of neural networks is not formed by corridors only (since then  phenomena such as the edge of stability \cite{cohen2021gradient} would not occur), there has been empirical evidence that regions where GD coincides with the GF occur in network training~\citep{cohen2021gradient,rosca2021discretization}. {\color{black}
Namely, for convergent learning rates, one observes that GF follows GD before the edge of stability phenomenon happens \cite{cohen2021gradient,rosca2023on},
 making corridors a possibly useful notion to understand the geometry of these regions (see Theorem \ref{thm:iterates_and_corridors}).
As discussed in the previous section, the CLR coincides the Polyak step-size, which has already been shown to be a fast and effective scheme in deep learning.
In particular, \citet{loizou2021stochastic} observe fast convergence for deep matrix factorization, multiple datasets from \citet{chang2011libsvm}, as well as for MNIST~\citep{lecun1989backpropagation}, CIFAR-10~\citep{cifar10}, CIFAR-100~\citep{cifar10}. Their results use the ResNet-34 \cite{resnets} and DenseNet-121 \cite{huang2017densenet} architectures. 
We add to this body of evidence by testing the CLR for the ResNet-18~\citep{resnets} and VGG~\citep{simonyan2014very} architectures on CIFAR-10, and a ResNet-50 architecture on ImageNet \cite{deng2009imagenet}. In all cases we tried, we  observed that the CLR provides stable training and quick convergence. Figure~\ref{fig:resnet_18} shows that the CLR reaches 0 training loss quicker than SGD over a learning rate sweep  on CIFAR-10, but that a larger learning rate reaches a better test set accuracy. Figure~\ref{fig:imagenet} shows ImageNet results across a series of batch sizes. Figure~\ref{fig:adaptive_lrs} shows the value of the CLR in training for both datasets. We note that with CLR we observe an (automated) initial learning rate warm-up followed by a decay, coinciding with the common practice of learning rate manually devised schedules~\citep{cosine_lr_decay,he2019bag}. We present additional experimental results in the Appendix \ref{section:additional_experiments}. 
}

\begin{figure}[ht]
\centering
\includegraphics[width=0.45\textwidth]{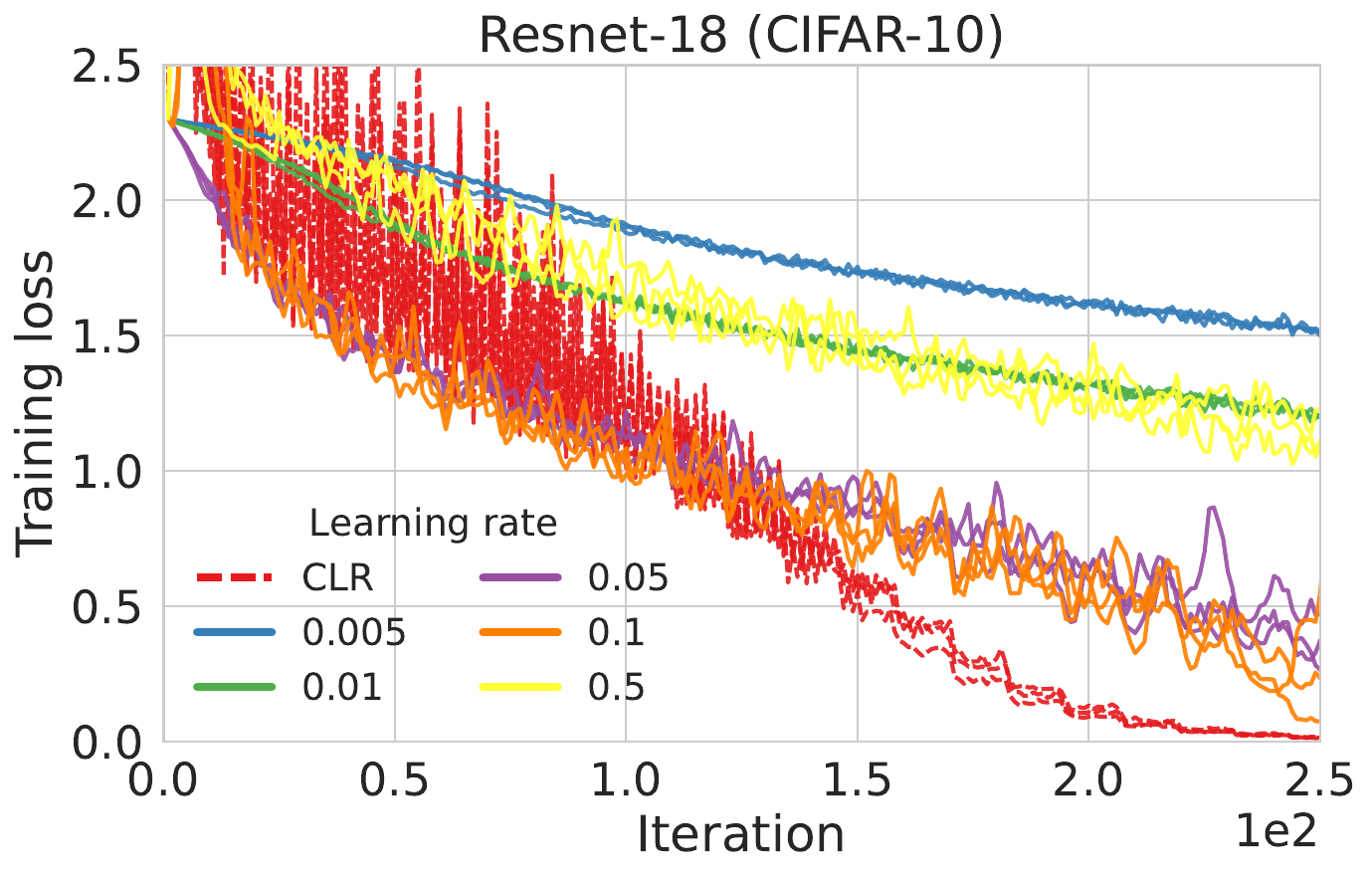}
\includegraphics[width=0.45\textwidth]{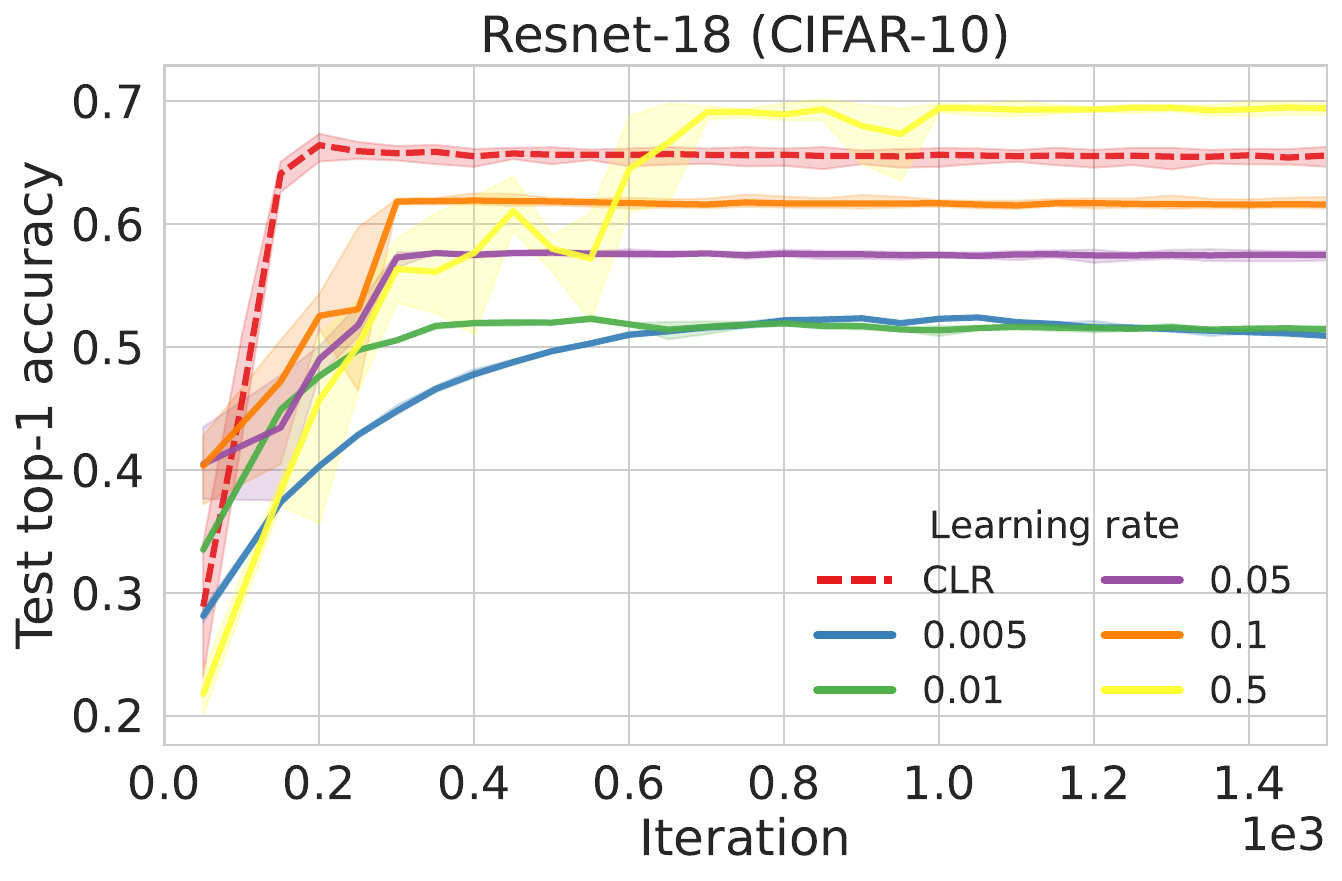}
\caption{The adaptive CLR converges for a ResNet-18 trained on CIFAR-10, and it does so quicker than SGD with a fixed learning rate obtained from a sweep. Batch size 4096. Results across a wide range of batch sizes can be found in Figure~\ref{fig:cifar_10_resnet_18_batch_size_sweep}. Results obtained using 3 seeds.}
\label{fig:resnet_18}
\end{figure}

\begin{figure}[ht]
\centering
\includegraphics[width=0.45\textwidth]{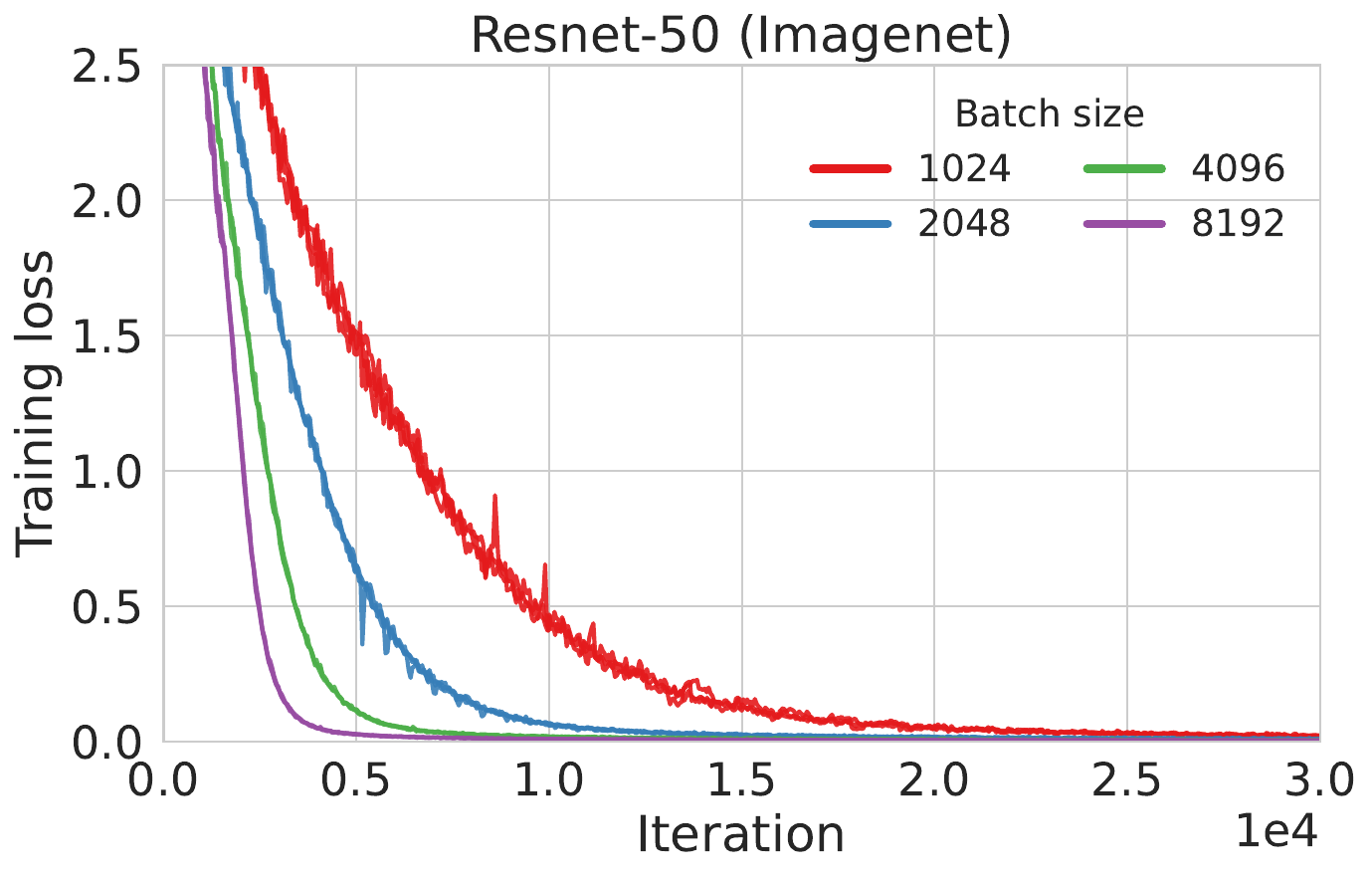}
\includegraphics[width=0.45\textwidth]{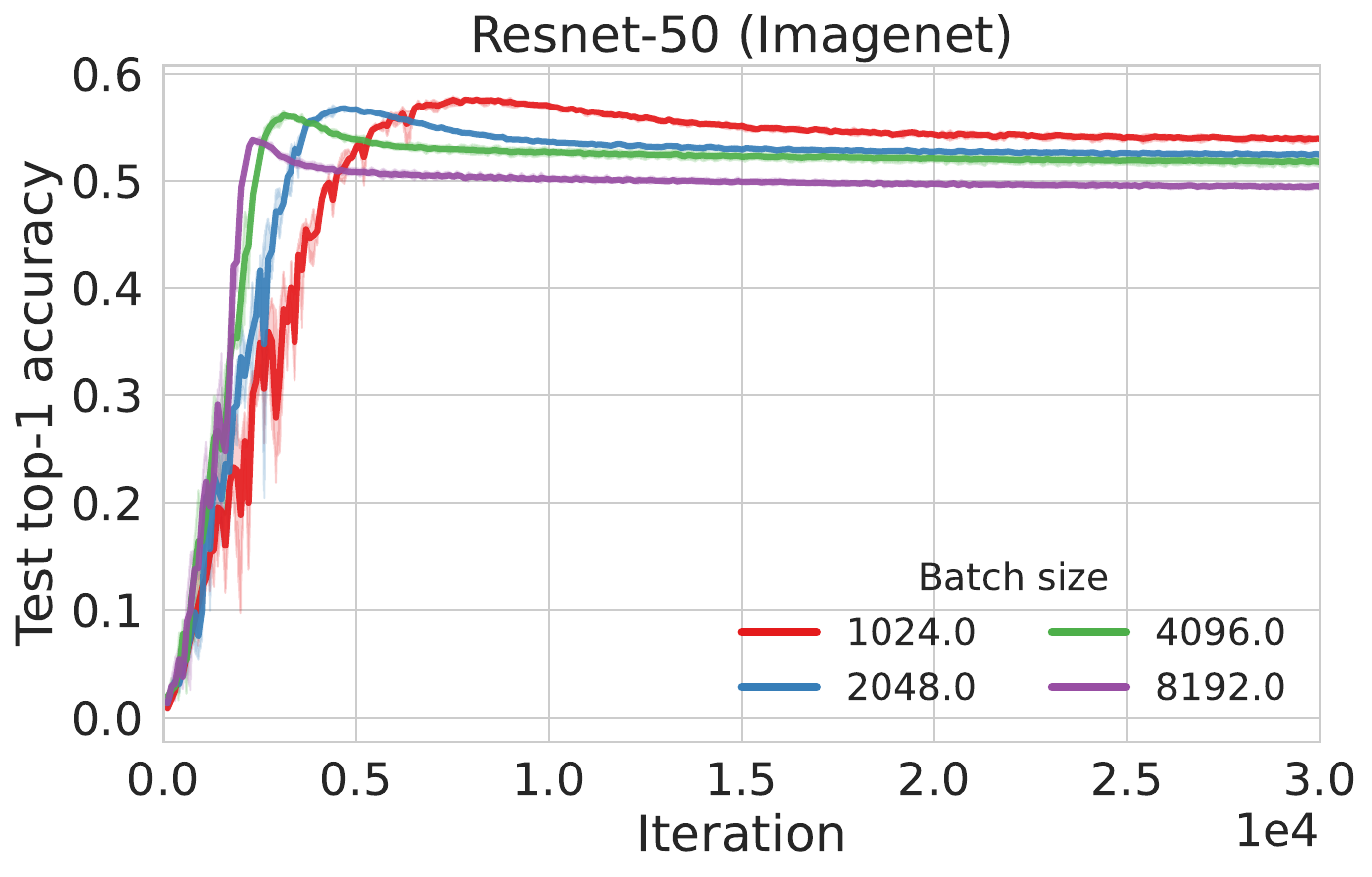}
\caption{The CLR converges for a ResNet-50 trained on Imagenet across batch sizes. We compare with vanilla SGD in Figure~\ref{fig:imagenet_comp}: the CLR converges quicker than the fixed learning rates. Results obtained using 3 seeds.}
\label{fig:imagenet}
\end{figure}

\begin{figure}[h!]
\centering
\includegraphics[width=0.45\textwidth]{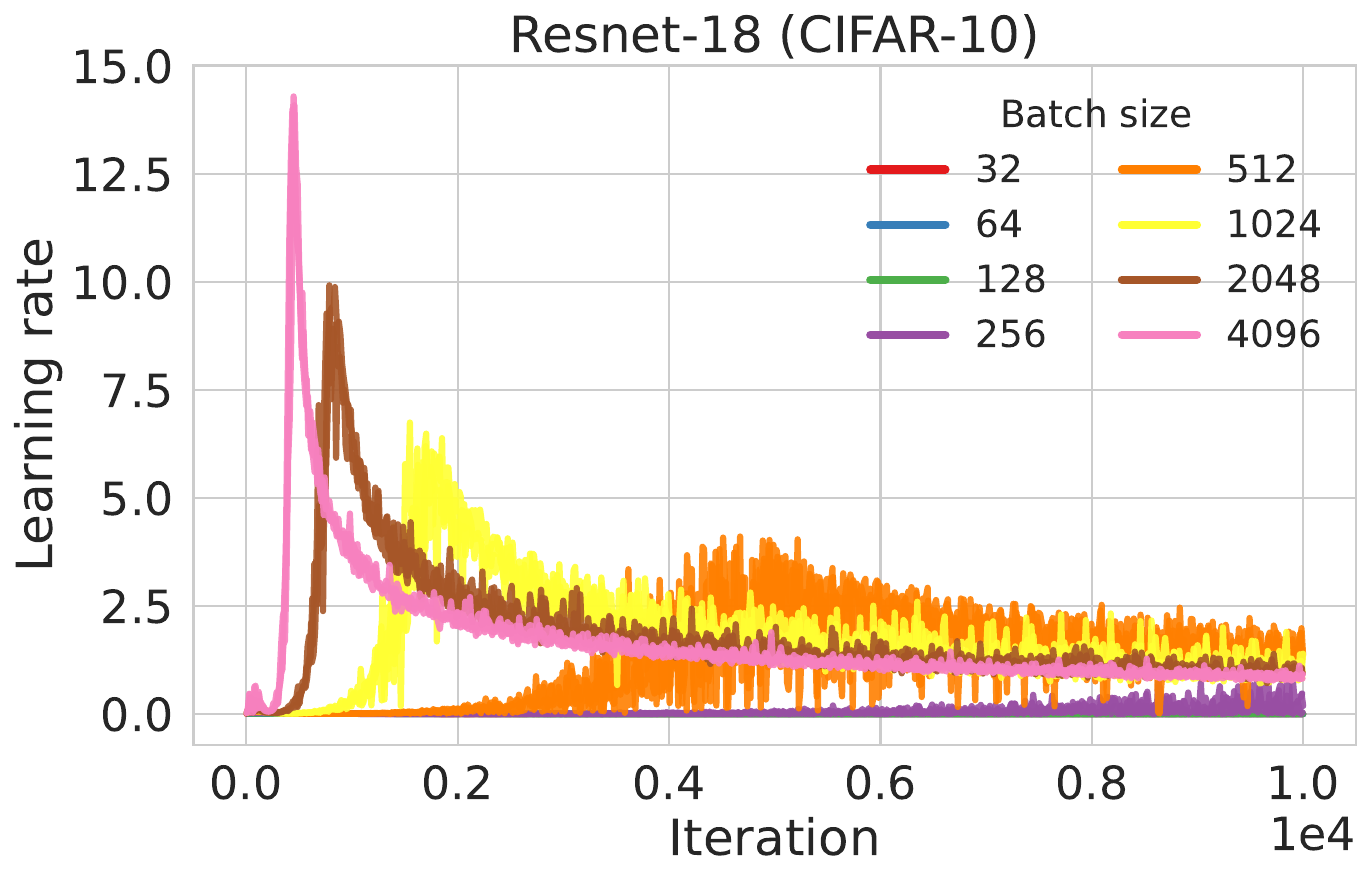}
\includegraphics[width=0.45\textwidth]{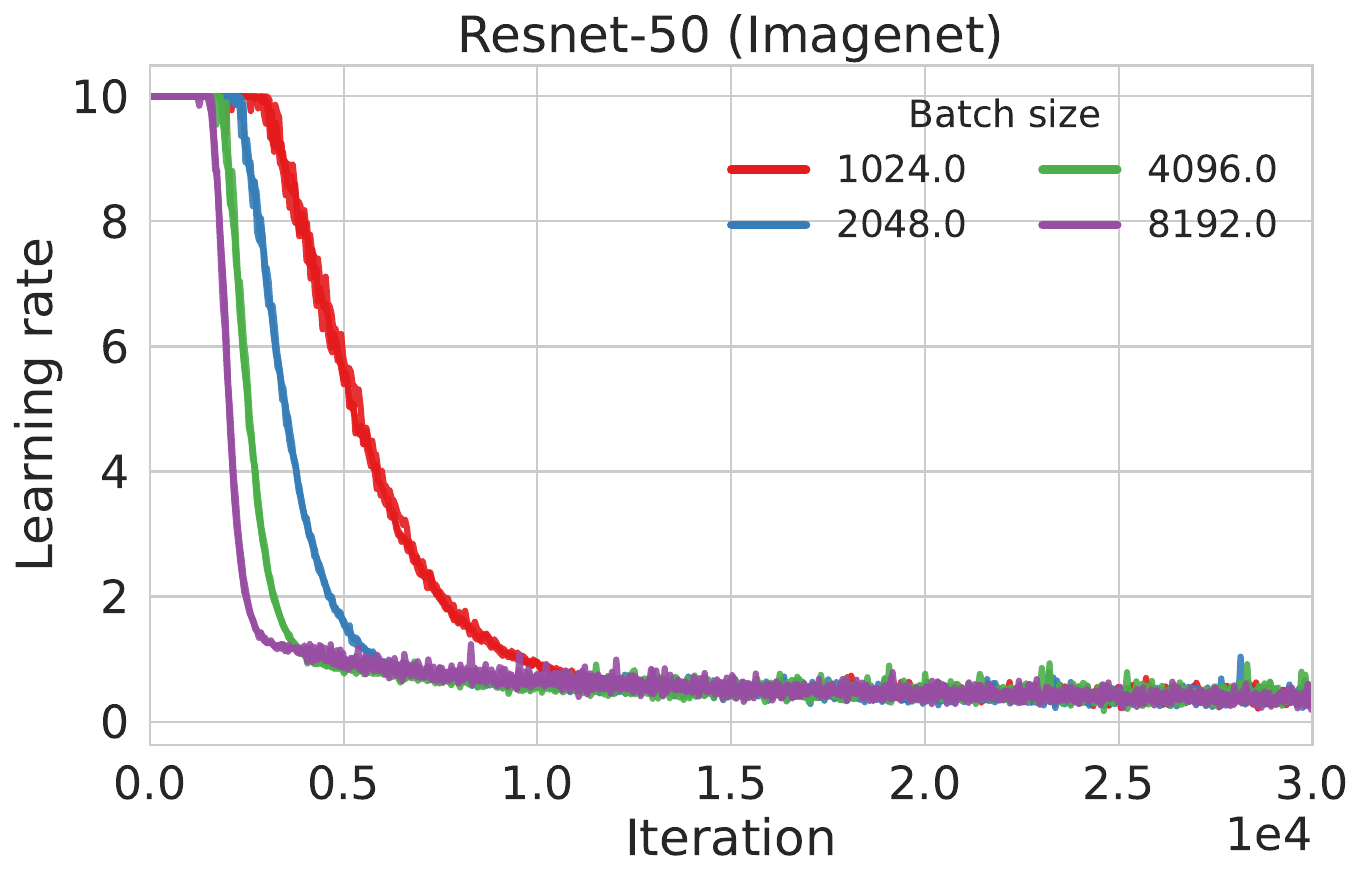}
\caption{CLR values adapt to training. Note: for ImageNet we cap the maximum value of the learning rate to 10. Results obtained using 3 seeds.}
\label{fig:adaptive_lrs}
\end{figure}

\newpage
\section{Conclusion}

We defined the notion of \textit{corridors} in loss landscapes as regions of the loss surface where the solutions of the GF become lines travelled at constant speed. We then found useful properties of corridors, by showing that inside corridors GD follows the GF exactly, that $H(\theta)g(\theta)$ vanishes, and that corridors are the only regions where this can happen. 
{\color{black}
We also showed that inside corridors the loss decreases linearly, allowing us to 
 construct an adaptive learning rate for GD, the corridor learning rate (CLT). We found that the CLT  coincides with the Polyak step-size  \cite{polyak1987introduction}; the Polyak step size had been observed to converge in the context of deep learning in prior work, and here we contributed to these results using additional architectures and datasets.
 }
The focus of this work is to advertise the notion of corridors, and foster new research along these geometric lines (pun intended!). Our focus here is not to devise an adaptation learning scheme competitive with algorithms such as Adam~\citep{kingma2014adam}, or other adaptive learning rates for SGD ~\citep{rosca2023on,ivgi2023dog,cutkosky2023mechanic} but instead to showcase the usefulness of studying corridors. In particular,
we hope that future work will hone into the theoretical and empirical study of \textit{when} corridors occur in neural network training.
Since inside corridors 
implicit regularization effects induced by the discretization error of GD do not occur, finding the scope of corridors in neural landscapes can further unlock the understanding of optimization induced implicit regularization in deep learning. 

{\color{black}
\section*{Acknowledgements}
We would like to thank David Barrett, Shakir Mohamed, and Michael Munn for helpful discussion and feedback as well as Patrick Cole for their support.
}
\newpage

\bibliographystyle{plainnat}
\bibliography{opt2023_corridors}

\begin{thebibliography}{34}
\providecommand{\natexlab}[1]{#1}
\providecommand{\url}[1]{\texttt{#1}}
\expandafter\ifx\csname urlstyle\endcsname\relax
  \providecommand{\doi}[1]{doi: #1}\else
  \providecommand{\doi}{doi: \begingroup \urlstyle{rm}\Url}\fi

\bibitem[lec(1989)]{lecun1989backpropagation}
Backpropagation applied to handwritten zip code recognition.
\newblock \emph{Neural computation}, 1\penalty0 (4):\penalty0 541--551, 1989.

\bibitem[Barrett and Dherin(2021)]{barrett2021implicit}
David~G.T. Barrett and Benoit Dherin.
\newblock Implicit gradient regularization.
\newblock In \emph{ICLR}, 2021.

\bibitem[Bottou et~al.(2018)Bottou, Curtis, and
  Nocedal]{bottou2018optimization}
Leon Bottou, Frank~E. Curtis, and Jorge Nocedal.
\newblock Optimization methods for large-scale machine learning.
\newblock \emph{SIAM Review}, 60\penalty0 (2), 2018.

\bibitem[Cattaneo et~al.(2023)Cattaneo, Klusowski, and
  Shigida]{cattaneo2023implicit}
Matias~D Cattaneo, Jason~M Klusowski, and Boris Shigida.
\newblock On the implicit bias of adam.
\newblock \emph{arXiv:2309.00079}, 2023.

\bibitem[Chang and Lin(2011)]{chang2011libsvm}
Chih-Chung Chang and Chih-Jen Lin.
\newblock {LIBSVM}: A library for support vector machines.
\newblock \emph{ACM Transactions on Intelligent Systems and Technology}, 2,
  2011.
\newblock Software available at \url{http://www.csie.ntu.edu.tw/~cjlin/libsvm}.

\bibitem[Cohen et~al.(2021)Cohen, Kaur, Li, Kolter, and
  Talwalkar]{cohen2021gradient}
Jeremy Cohen, Simran Kaur, Yuanzhi Li, J~Zico Kolter, and Ameet Talwalkar.
\newblock Gradient descent on neural networks typically occurs at the edge of
  stability.
\newblock In \emph{ICLR}, 2021.

\bibitem[Cutkosky et~al.(2023)Cutkosky, Defazio, and
  Mehta]{cutkosky2023mechanic}
Ashok Cutkosky, Aaron Defazio, and Harsh Mehta.
\newblock Mechanic: A learning rate tuner.
\newblock \emph{arXiv:2306.00144}, 2023.

\bibitem[Deng et~al.(2009)Deng, Dong, Socher, Li, Li, and
  Fei-Fei]{deng2009imagenet}
Jia Deng, Wei Dong, Richard Socher, Li-Jia Li, Kai Li, and Li~Fei-Fei.
\newblock Imagenet: A large-scale hierarchical image database.
\newblock In \emph{CVPR}, 2009.

\bibitem[Elkabetz and Cohen(2021)]{elkabetz2021continuous_vs_discrete}
Omer Elkabetz and Nadav Cohen.
\newblock Continuous vs. discrete optimization of deep neural networks.
\newblock In \emph{NeurIPS}, 2021.

\bibitem[Ghosh et~al.(2023)Ghosh, Lyu, Zhang, and Wang]{ghosh2023implicit}
Avrajit Ghosh, He~Lyu, Xitong Zhang, and Rongrong Wang.
\newblock Implicit regularization in heavy-ball momentum accelerated stochastic
  gradient descent.
\newblock \emph{ICLR}, 2023.

\bibitem[Gunasekar et~al.(2017)Gunasekar, Woodworth, Bhojanapalli, Neyshabur,
  and Srebro]{gunasekar2017implicit}
Suriya Gunasekar, Blake~E Woodworth, Srinadh Bhojanapalli, Behnam Neyshabur,
  and Nati Srebro.
\newblock Implicit regularization in matrix factorization.
\newblock \emph{NeurIPS}, 2017.

\bibitem[Hairer et~al.(2006)Hairer, Hochbruck, Iserles, and
  Lubich]{hairer2006geometric}
Ernst Hairer, Marlis Hochbruck, Arieh Iserles, and Christian Lubich.
\newblock Geometric numerical integration.
\newblock \emph{Oberwolfach Reports}, 3\penalty0 (1):\penalty0 805--882, 2006.

\bibitem[Hazan and Kakade(2019)]{hazan2019revisiting}
Elad Hazan and Sham Kakade.
\newblock Revisiting the polyak step size.
\newblock \emph{arXiv:1905.00313}, 2019.

\bibitem[He et~al.(2016)He, Zhang, Ren, and Sun]{resnets}
Kaiming He, Xiangyu Zhang, Shaoqing Ren, and Jian Sun.
\newblock Deep residual learning for image recognition.
\newblock In \emph{CVPR}, 2016.

\bibitem[He et~al.(2019)He, Zhang, Zhang, Zhang, Xie, and Li]{he2019bag}
Tong He, Zhi Zhang, Hang Zhang, Zhongyue Zhang, Junyuan Xie, and Mu~Li.
\newblock Bag of tricks for image classification with convolutional neural
  networks.
\newblock In \emph{CVPR}, 2019.

\bibitem[Huang et~al.(2017)Huang, Liu, van~der Maaten, and
  Weinberger]{huang2017densenet}
Gao Huang, Zhuang Liu, Laurens van~der Maaten, and Kilian~Q. Weinberger.
\newblock Densely connected convolutional networks.
\newblock In \emph{CVPR}, 2017.

\bibitem[Ivgi et~al.(2023)Ivgi, Hinder, and Carmon]{ivgi2023dog}
Maor Ivgi, Oliver Hinder, and Yair Carmon.
\newblock Dog is sgd's best friend: A parameter-free dynamic step size
  schedule.
\newblock \emph{arXiv:2302.12022}, 2023.

\bibitem[Kingma and Ba(2015)]{kingma2014adam}
Diederik~P Kingma and Jimmy Ba.
\newblock Adam: A method for stochastic optimization.
\newblock In \emph{ICLR}, 2015.

\bibitem[Krizhevsky(2009)]{cifar10}
Alex Krizhevsky.
\newblock Learning multiple layers of features from tiny images.
\newblock
  \emph{https://www.cs.toronto.edu/~kriz/learning-features-2009-TR.pdf}, 2009.

\bibitem[Kunin et~al.(2021)Kunin, Sagastuy-Brena, and Ganguli]{symmetry}
Daniel Kunin, Javier Sagastuy-Brena, and Hidenori~Tanaka Ganguli, Surya Daniel
  L.K.~Yamins.
\newblock Symmetry, conservation laws, and learning dynamics in neural
  networks.
\newblock In \emph{ICLR}, 2021.

\bibitem[Loizou et~al.(2021)Loizou, Vaswani, Laradji, and
  Lacoste-Julien]{loizou2021stochastic}
Nicolas Loizou, Sharan Vaswani, Issam Laradji, and Simon Lacoste-Julien.
\newblock Stochastic polyak step-size for sgd: An adaptive learning rate for
  fast convergence.
\newblock In \emph{AISTATS}, 2021.

\bibitem[Loshchilov and Hutter(2017)]{cosine_lr_decay}
Ilya Loshchilov and Frank Hutter.
\newblock Sgdr: Stochastic gradient descent with warm restarts.
\newblock In \emph{ICLR}, 2017.

\bibitem[Ma et~al.(2022)Ma, Kunin, Wu, and Ying]{ma2022quadratic}
Chao Ma, Daniel Kunin, Lei Wu, and Lexing Ying.
\newblock Beyond the quadratic approximation: the multiscale structure of
  neural network loss landscapes, 2022.

\bibitem[Miyagawa(2022)]{miyagawatoward}
Taiki Miyagawa.
\newblock Toward equation of motion for deep neural networks: Continuous-time
  gradient descent and discretization error analysis.
\newblock In \emph{NeurIPS}, 2022.

\bibitem[Orvieto et~al.(2022)Orvieto, Lacoste-Julien, and
  Loizou]{orvieto2022dynamics}
Antonia Orvieto, Simon Lacoste-Julien, and Nicolas Loizou.
\newblock Dynamics of sgd with stochastic polyak stepsizes: Truly adaptive
  variants and convergence to exact solution.
\newblock In \emph{NeurIPS}, 2022.

\bibitem[Polyak(1964)]{polyak1964some}
Boris~T Polyak.
\newblock Some methods of speeding up the convergence of iteration methods.
\newblock \emph{Ussr computational mathematics and mathematical physics},
  4\penalty0 (5):\penalty0 1--17, 1964.

\bibitem[Polyak(1987)]{polyak1987introduction}
Boris~T Polyak.
\newblock Introduction to optimization. optimization software.
\newblock \emph{Inc., Publications Division, New York}, 1:\penalty0 32, 1987.

\bibitem[Rosca and Deisenroth(2023)]{rosca2023implicit}
Mihaela Rosca and Marc~Peter Deisenroth.
\newblock Implicit regularisation in stochastic gradient descent: from
  single-objective to two-player games.
\newblock \emph{arXiv:2307.05789}, 2023.

\bibitem[Rosca et~al.(2021)Rosca, Wu, Dherin, and
  Barrett]{rosca2021discretization}
Mihaela Rosca, Yan Wu, Benoit Dherin, and David~G.T. Barrett.
\newblock Discretization drift in two-player games.
\newblock In \emph{ICML}, 2021.

\bibitem[Rosca et~al.(2023)Rosca, Wu, Qin, and Dherin]{rosca2023on}
Mihaela Rosca, Yan Wu, Chongli Qin, and Benoit Dherin.
\newblock On a continuous time model of gradient descent dynamics and
  instability in deep learning.
\newblock In \emph{TMLR}, 2023.

\bibitem[Simonyan and Zisserman(2015)]{simonyan2014very}
Karen Simonyan and Andrew Zisserman.
\newblock Very deep convolutional networks for large-scale image recognition.
\newblock In \emph{ICLR}, 2015.

\bibitem[Smith et~al.(2021)Smith, Dherin, Barrett, and De]{smith2021on}
Samuel~L Smith, Benoit Dherin, David~G.T. Barrett, and Soham De.
\newblock On the origin of implicit regularization in stochastic gradient
  descent.
\newblock In \emph{ICLR}, 2021.

\bibitem[Soudry et~al.(2018)Soudry, Hoffer, Nacson, Gunasekar, and
  Srebro]{soudry2018implicit}
Daniel Soudry, Elad Hoffer, Mor~Shpigel Nacson, Suriya Gunasekar, and Nathan
  Srebro.
\newblock The implicit bias of gradient descent on separable data.
\newblock \emph{JMLR}, 2018.

\bibitem[Stolfi et~al.(2004)Stolfi, Eppstein, Matthews, Gond, Innocenti, Nead,
  and Collaborators]{wiki_ruled_surface}
Jorge Stolfi, David Eppstein, Charles Matthews, Paul Gond, Luca Innocenti, Sam
  Nead, and Anonymous Collaborators.
\newblock Ruled surface.
\newblock \emph{Wikipedia}, 2004.
\newblock URL \url{https://en.wikipedia.org/wiki/Ruled_surface}.

\end{thebibliography}

\newpage
\clearpage
\appendix

\section{Beyond gradient descent}
\label{app:beyond_gd}
\subsection{Momentum follows a sped up GF in corridors}\label{section:momentum}

In this section, we investigate the behaviour of momentum inside corridors. As in the main text, we will use that inside a corridor, we have $g(\theta - t g(\theta)) = g(\theta), \forall t \ge 0$ (Lemma~\ref{lemma:cost_grad}).

We now use this property to investigate momentum updates in corridors. Momentum updates with decay rates $\beta$ and learning rate $h$ are defined as 
\begin{align}
    v_0 &= 0 \\
    v_t &= \beta v_{t-1} - h  g(\theta_{t-1}) \\
    \theta_t &= \theta_{t-1} + v_t
\end{align}

\begin{lemma} Momentum updates in a corridor (the model is initialised in a corridor and remains in a corridor as updates occur) can be written as 
\begin{align}
  v_t &= - h \frac{(1 - \beta^{t})}{1-\beta} g(\theta_0) \\
  \theta_t &= \theta_0 - \frac{h}{1-\beta} \sum_{i=1}^{t} (1 - \beta^{i}) g(\theta_0) \label{eq:mom}
\end{align}
\end{lemma}

\begin{proof}
The proof is by induction. Base case:
\begin{align}
    v_1 &= \beta v_0 - h g(\theta_0) =  - h g(\theta_0) \\
    \theta_1 &= \theta_0 - h g(\theta_0)
\end{align}

Induction step:
\begin{eqnarray*}
    v_t &=& \beta v_{t-1} - h g(\theta_{t-1})  \\
       &=& \beta \left(- h \frac{(1 - \beta^{t-1})}{1-\beta} g(\theta_0)\right) - h g(\theta_{t-1}) \\
       &=& \beta \left(- h \frac{(1 - \beta^{t-1})}{1-\beta} g(\theta_0)\right) - h g(\theta_{0}) \quad\textrm{(Induct. Hypo. \&}\, \textrm{ Lemma~\ref{lemma:cost_grad})} \\
       &=& - h \frac{1}{1-\beta} \left(\beta (1 - \beta^{t-1}) + 1 -\beta \right) g(\theta_{0}) \\
       &=& - h \frac{(1 - \beta^{t})}{1-\beta} g(\theta_0)
       \label{eq:v_t_eq}
\end{eqnarray*}
From here:
\begin{align}
   \theta_t &= \theta_{t-1} + \sum_{i=1}^{t} v_i =  \theta_0 - \frac{h}{1-\beta} \sum_{i=1}^{t} (1 - \beta^{i}) g(\theta_0)
\end{align}
finishing the proof.
\end{proof}

\begin{remark}
The above results show that if momentum enters a corridor with a zero velocity vector, it follows the lines of the GF, but at a different speed  given by the coefficient $\beta$. Note that the assumption that momentum enters the corridor with zero velocity (i.e, $v_0 = 0$) is key to the proof. A similar proof can be repeated by assuming the initial velocity to be aligned to the initial gradient corridor (i.e., $v_0 = \lambda g(\theta_0)$) allowing us to use Lemma \ref{lemma:cost_grad}. However, outside of these conditions, momentum may no longer stay on the lines of steepest descent and may start drifting away from the GF. This gives an edge to momentum compared to GD since implicit regularization  \cite{barrett2021implicit,rosca2023on} can occur.
\end{remark}

\begin{remark} Inside corridors momentum with initial zero velocity follows the ODE \textit{exactly}:
\begin{align}
\dot{\theta} = - \frac{1 - \beta^n}{1-\beta} g(\theta(t)), \hspace{2em} t \in [hn, h(n+1)] 
\end{align}
\end{remark}
\begin{proof}
Since this ODE is of the form $\dot{\theta} = - c_n g(\theta(t)), \hspace{2em} t \in [hn, h(n+1)]$, we have that just like under the GF, under this ODE $\dot{g(\theta)} = - c_n H(\theta(t)) g(\theta) = 0$, and thus as long as a corridor is not exited, under the trajectory of this ODE the gradient is constant in time. 
We then have:
\begin{align}
    \theta(n h) = \theta((n-1) h) - h \frac{1 - \beta^n}{1-\beta} g(\theta(t))
    =  \theta((n-1) h) - h \frac{1 - \beta^n}{1-\beta} g(\theta(0))
\end{align}
which when expanded in $n$ leads to the momentum equation in corridors in Eq~\ref{eq:mom}.
This results shows that the implicit regularization terms found for momentum by ~\citet{ghosh2023implicit} vanish, just like for GD, and that momentum follows the same trajectory as GD, but at a higher speed (given by the coefficient $\beta$).
\end{proof}

\subsection{RK4 follows the GF in corridors}\label{section:rk4}
In this section, we investigate the behaviour of RK4 inside corridors. 

We have that the RK4 update is 
\begin{align}
    k_1 &= - g(\theta_{t-1}) \\
    k_2 &= - g(\theta_{t-1} + \frac h 2 k_1) \\
    k_3 &= - g(\theta_{t-1} + \frac h 2 k_2) \\
    k_4 &= - g(\theta_{t-1} + h k_3) \\
    \theta_{t} &= \theta_{t-1} + \frac{1}{6}(k_1 + 2 k_2 + 2 k_3 + k_4)
\end{align}

Since we have that inside a corridor $g(\theta - t g(\theta)) = g(\theta), \forall t \ge 0$, it follows that $k_1 = k_2 = k_3 = k_4 = - g(\theta_{t-1})$, and thus $\theta_{t} = \theta_{t-1} - g(\theta_{t-1})$, which is exactly the GF trajectory inside a corridor. As with GD, we find that Rk4 follows the GF exactly inside corridors.

\subsection{Stochastic Gradient Descent}\label{section:sgd}

We now study the implicit regularization effect of SGD inside corridors. \citet{smith2021on}
show that after one epoch SGD iterates, in expectation over all possible batch shufflings $\sigma$, follow a modified flow with an error $\mathcal{O}(m^3h^3)$:
\begin{align}
    E_{\sigma}\left[\theta_m \right] &= \theta(mh) + \mathcal{O}(m^3h^3), \hspace{3em} \text{with} \\
    \dot{\theta} &= - g(\theta) - \frac{h}{4} \nabla_{\theta} \|g(\theta)\|^2 -  \frac{h}{4m} \sum_{i=1}^{m}\nabla_{\theta} \|g(\theta) - g_i(\theta) \|^2
\end{align}
where $m$ is the number of updates in an epoch and $g_i$ is the gradient corresponding to batch $i$.

Since $\frac{1}{2}\nabla_{\theta} \|g(\theta)\|^2 = H(\theta)g(\theta)$, and $H(\theta)g(\theta)=0$ inside a corridor (Theorem~\ref{thm:corridor}), this entails that the only regularization pressure of order $\mathcal{O}(h^2)$ left from SGD is the minimisation of $\nabla_{\theta} \|g(\theta) - g_i(\theta) \|^2$.

\section{Additional experimental results}\label{section:additional_experiments}

\begin{figure}[ht]
\centering
\includegraphics[width=0.49\textwidth]{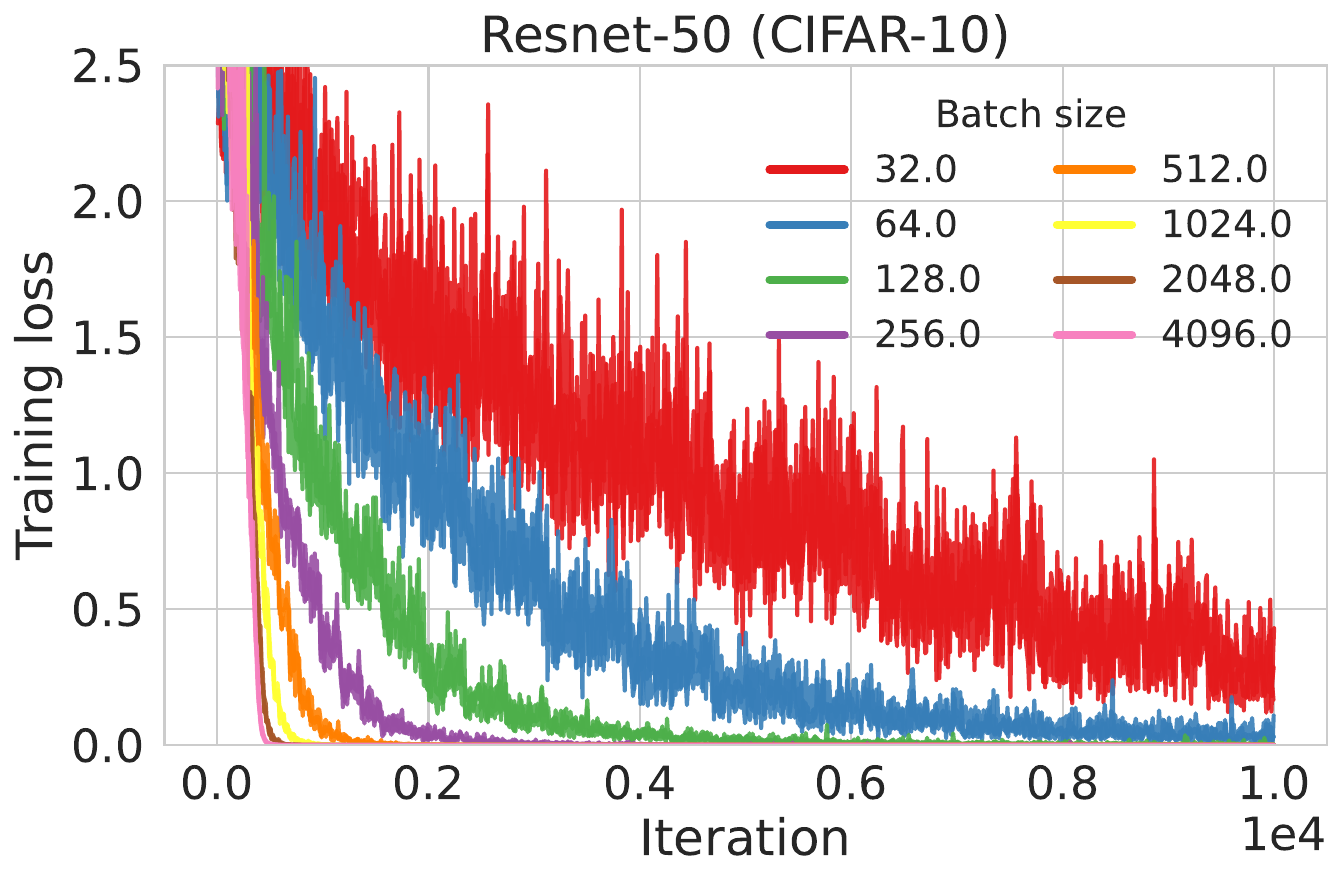}
\includegraphics[width=0.49\textwidth]{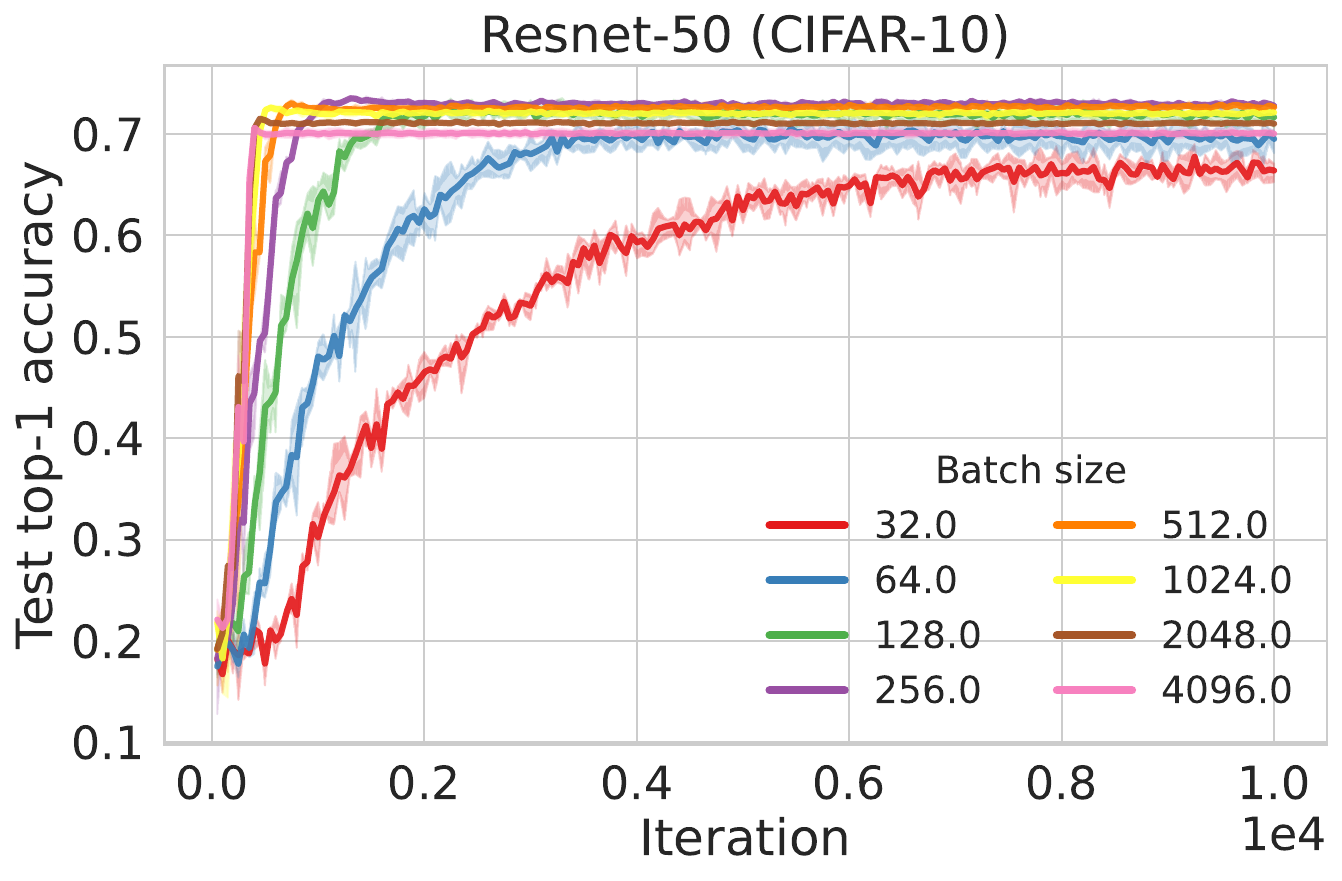}
\caption{The corridor learning rate converges for a ResNet-50 trained on CIFAR-10.}
\end{figure}

\begin{figure}[ht]
\centering
\includegraphics[width=0.49\textwidth]{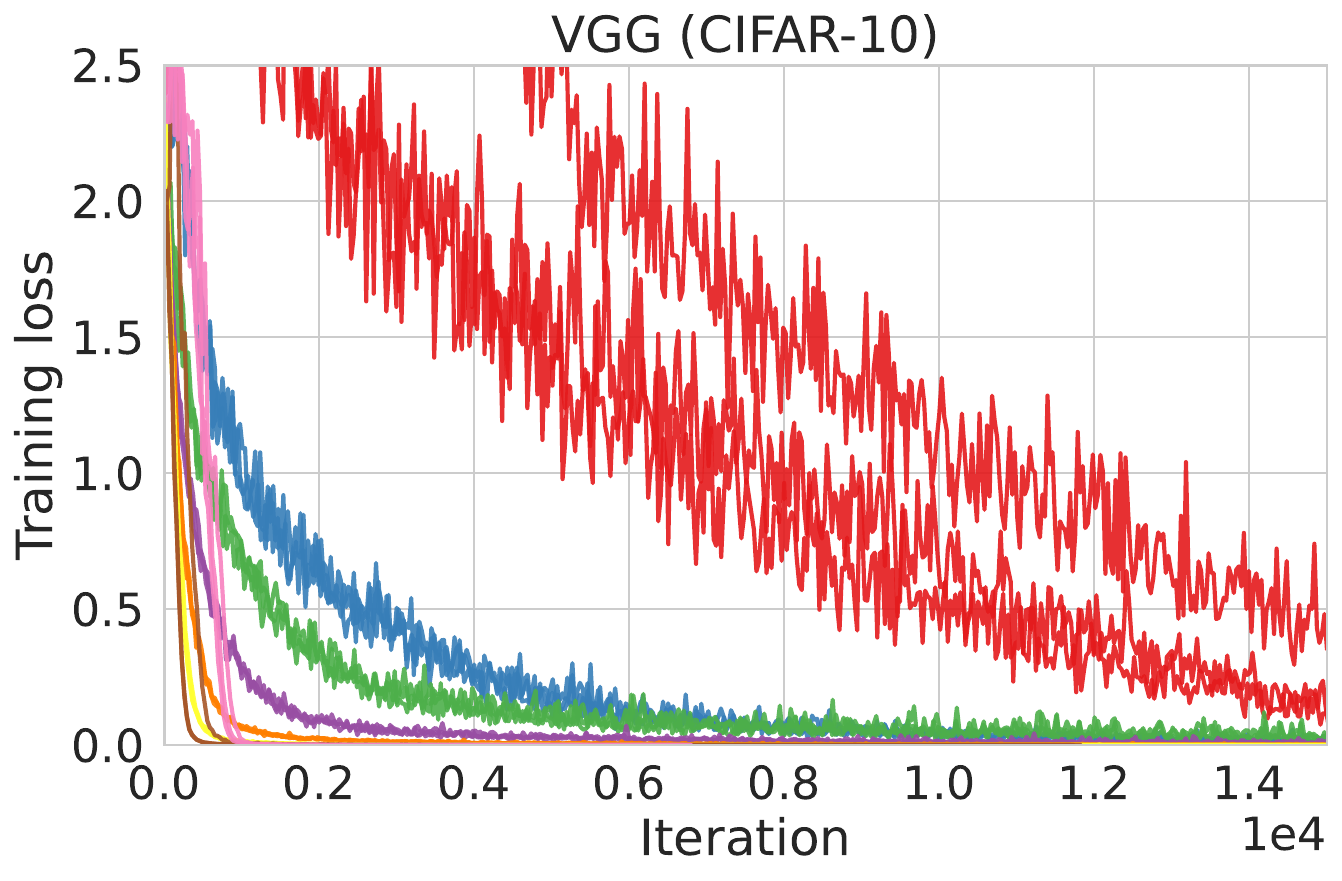}
\includegraphics[width=0.49\textwidth]{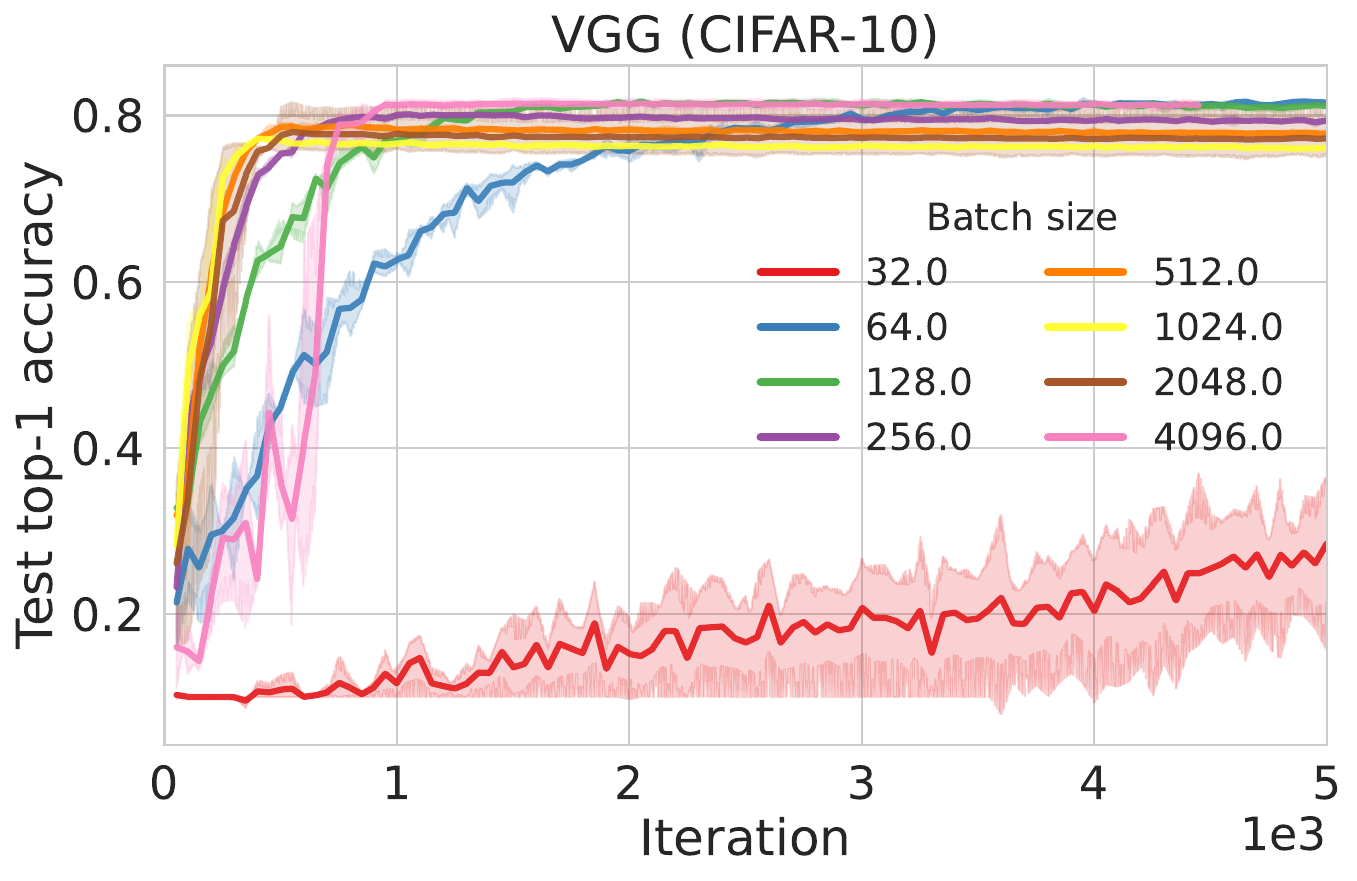}
\caption{The CLR converges for a VGG trained on CIFAR-10. Results obtained using 3 seeds.}
\end{figure}

\begin{figure}[ht]
\centering
\includegraphics[width=0.32\textwidth]{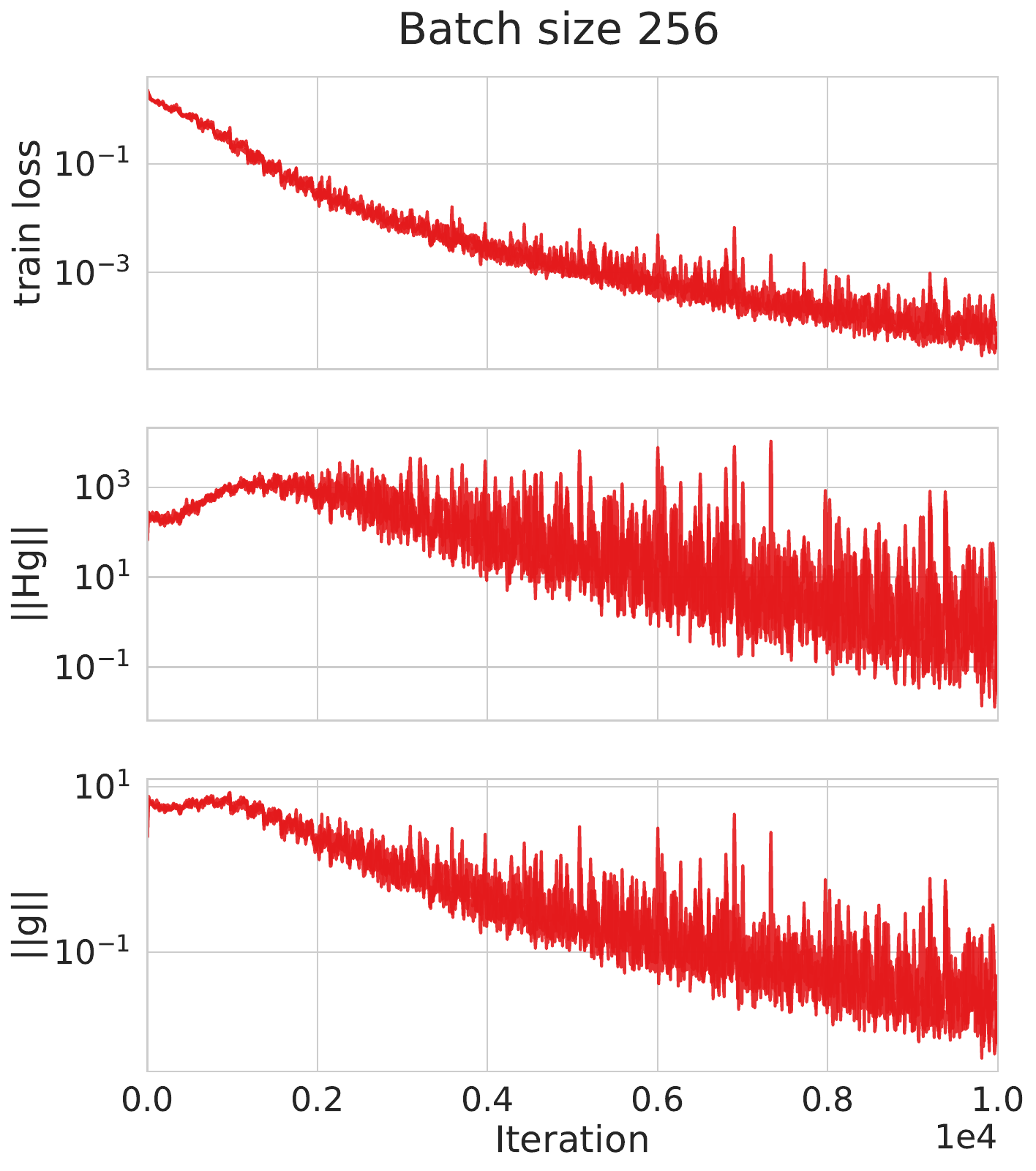}
\includegraphics[width=0.32\textwidth]{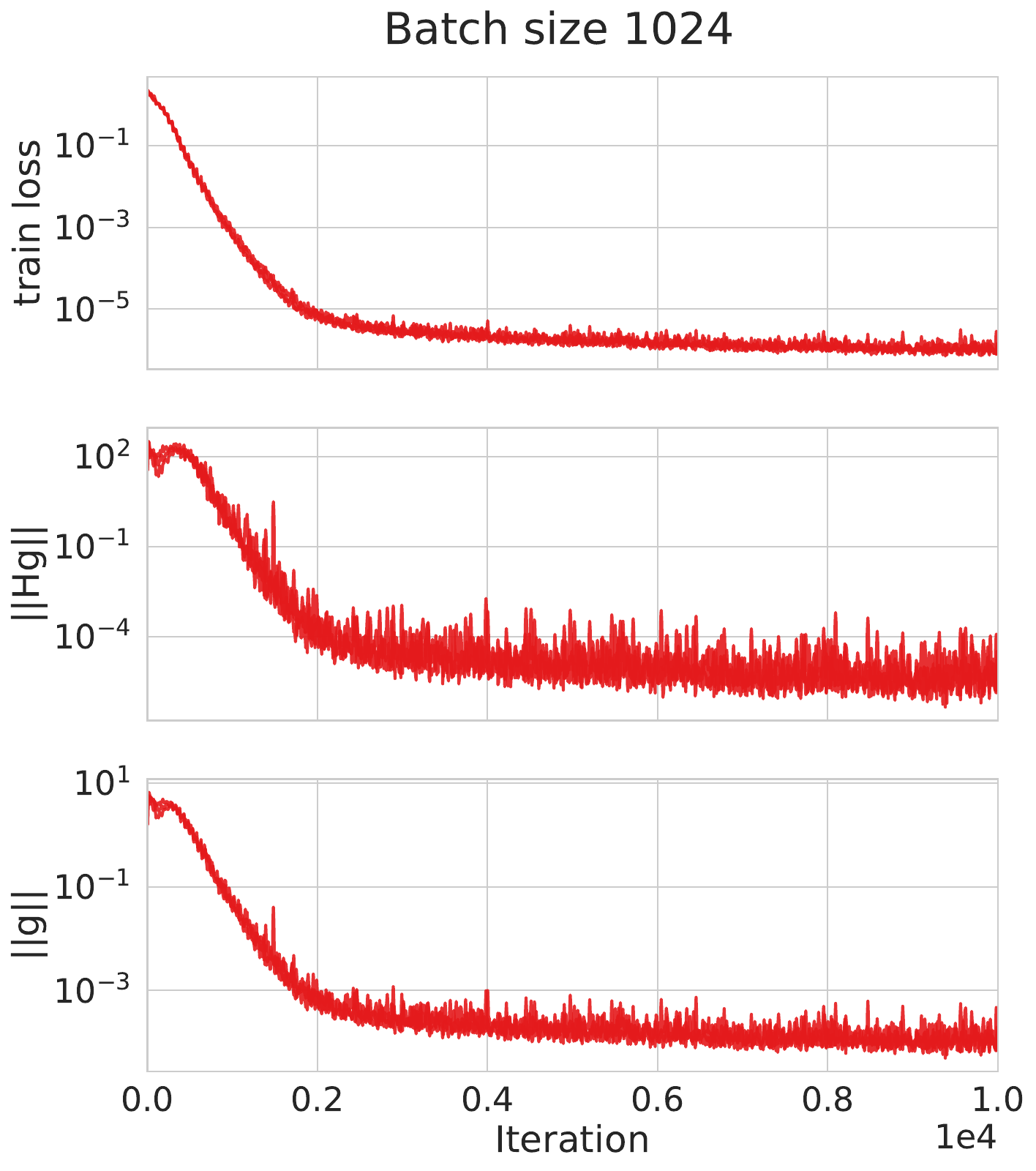}
\includegraphics[width=0.32\textwidth]{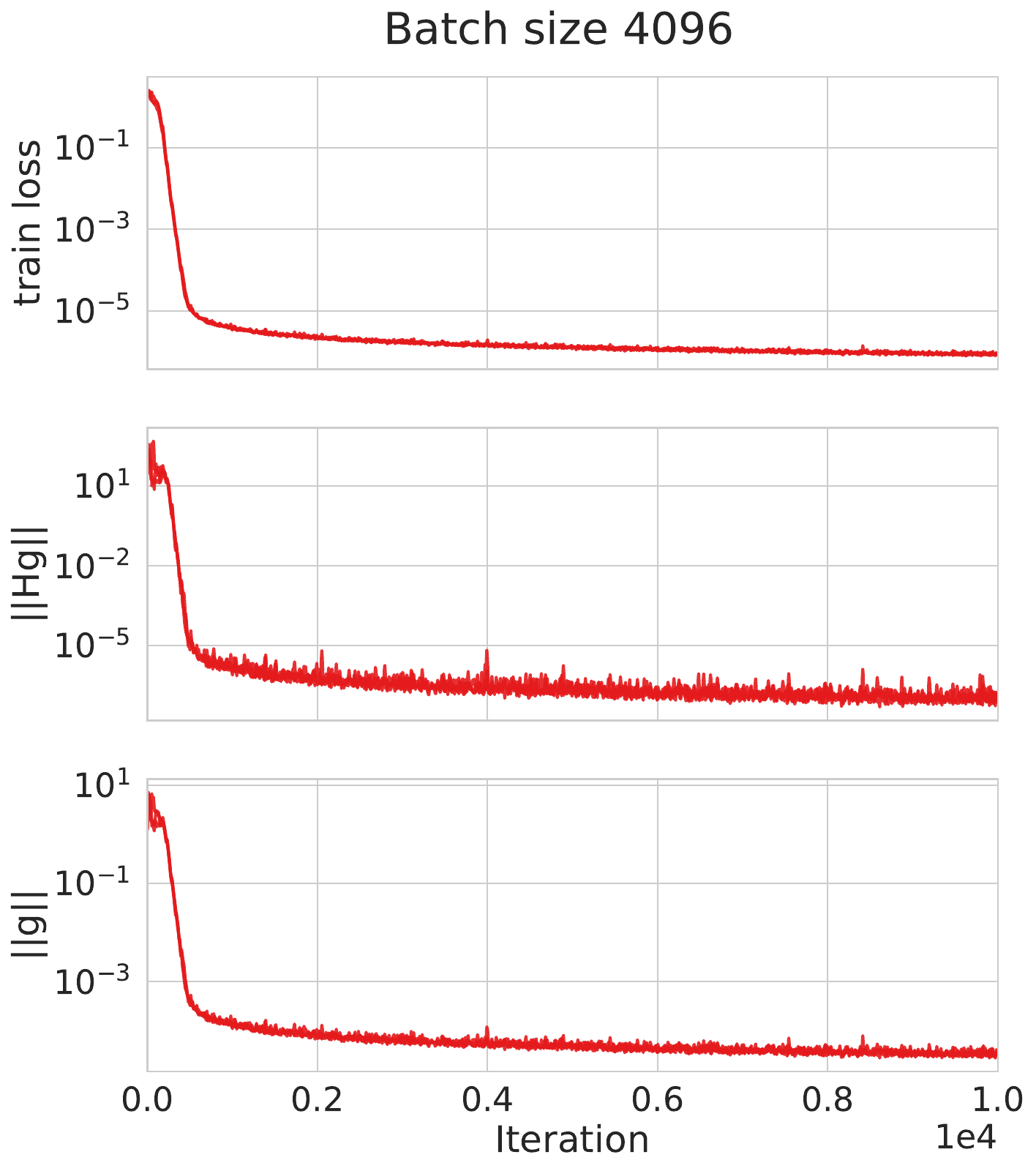}
\caption{Loss and $H(\theta)g(\theta)$ across a trajectory obtained with the learning rate adaptation scheme for a ResNet-18 trained on CIFAR-10. Results obtained using 3 seeds.}
\label{fig:cifar_resnet_18_hg}
\end{figure}

\begin{figure}[ht]
\centering
\includegraphics[width=0.49\textwidth]{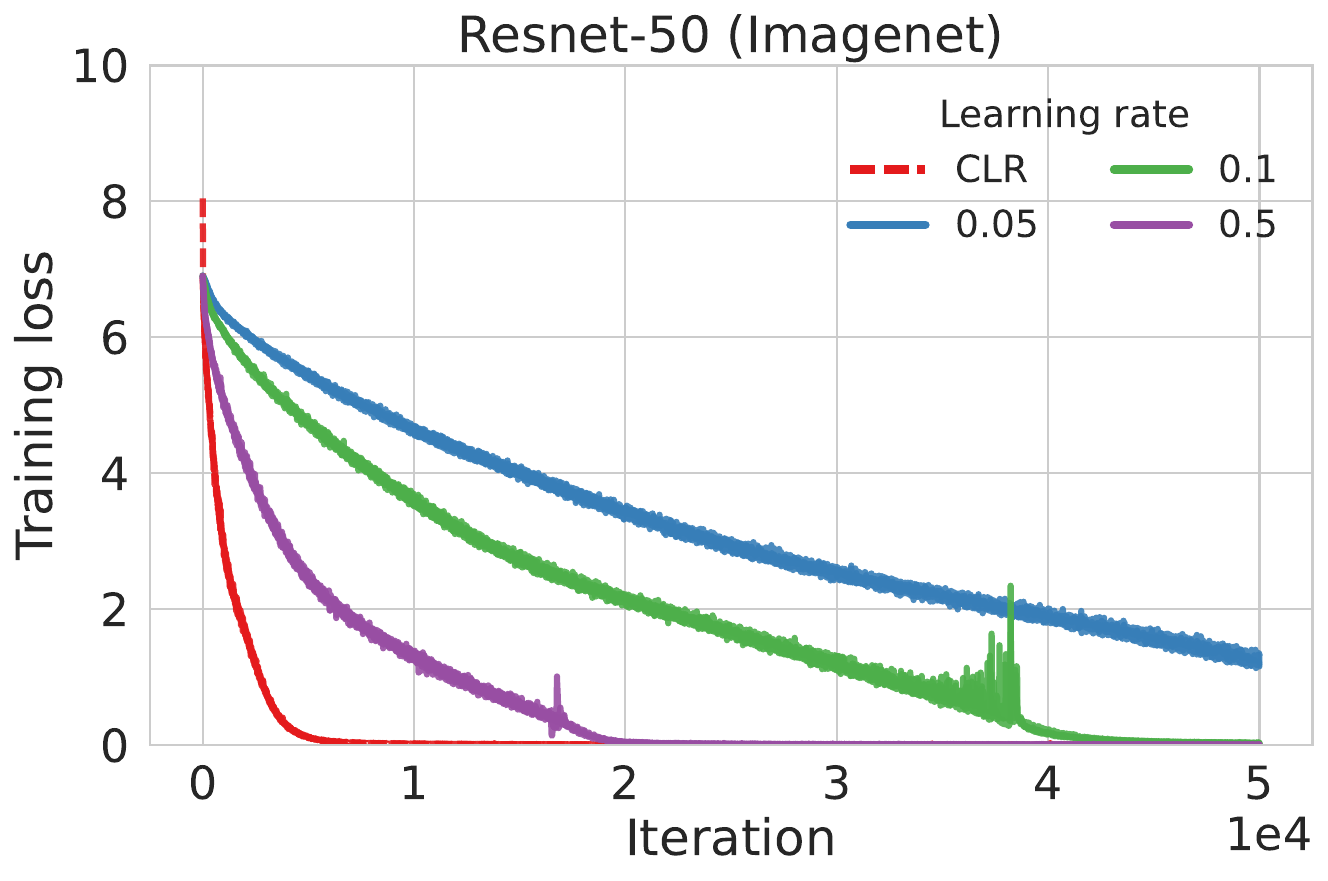}
\includegraphics[width=0.49\textwidth]{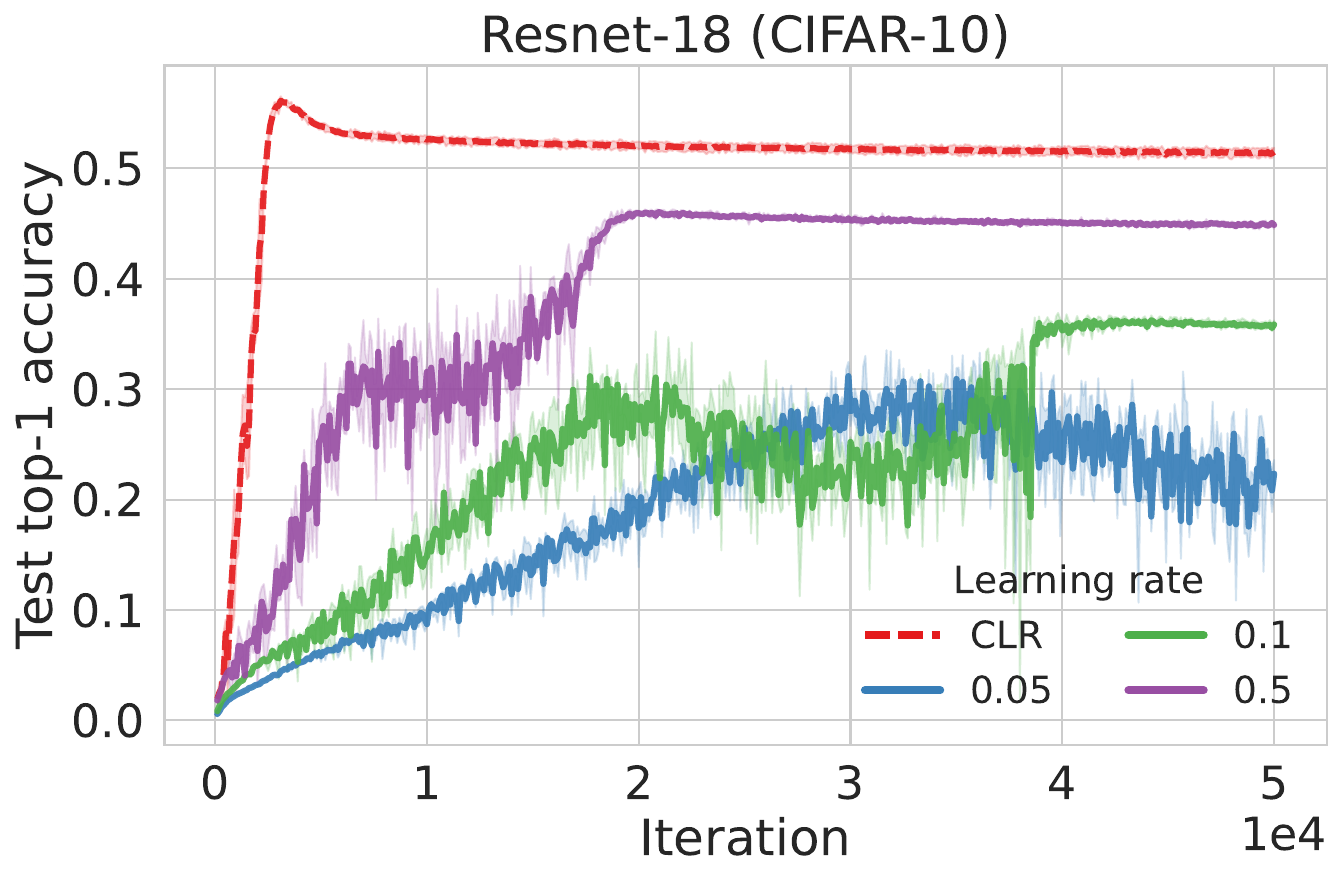}
\caption{The CLR converges for a ResNet-50 trained on Imagenet, and it does so quicker that vanilla SGD (curves with fixed learning rates). We show that the same setting the CLR converges across a wide range of batch sizes in Figure~\ref{fig:cifar_10_resnet_18_batch_size_sweep}. Results obtained using 3 seeds.}
\label{fig:imagenet_comp}
\end{figure}

\begin{figure}[ht]
\centering
\includegraphics[width=0.49\textwidth]{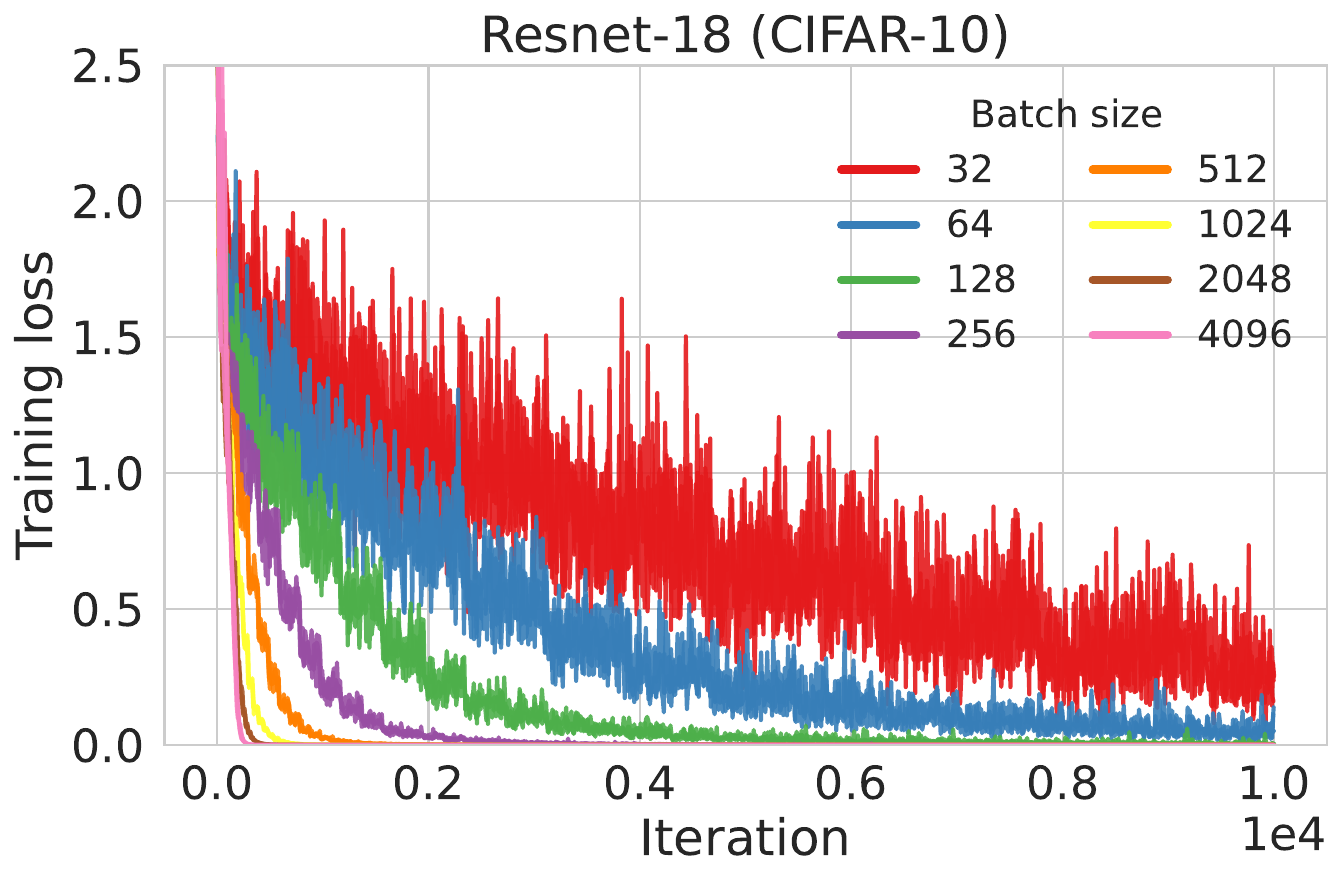}
\includegraphics[width=0.49\textwidth]{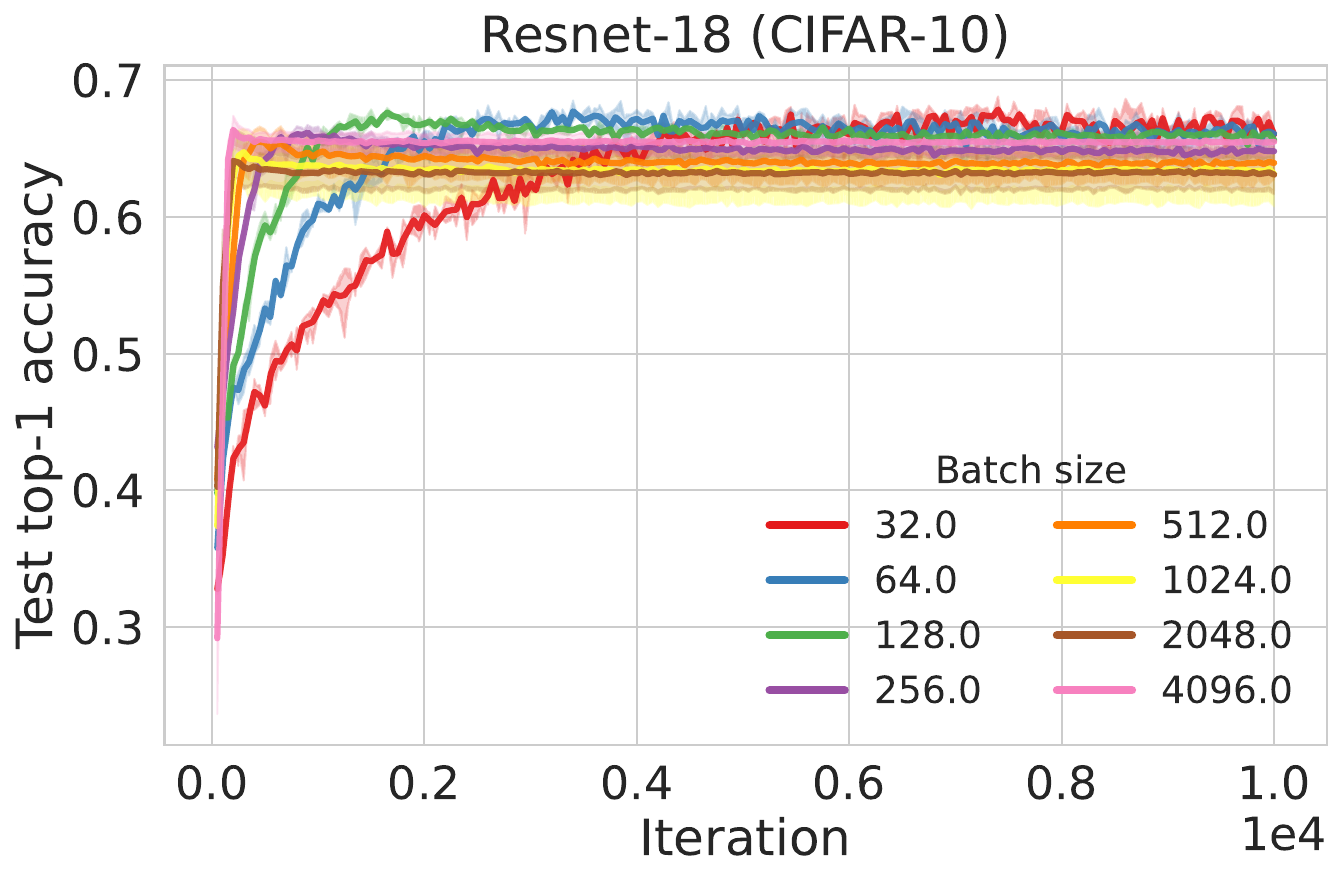}
\caption{The CLR converges for a ResNet-18 trained on CIFAR10.}
\label{fig:cifar_10_resnet_18_batch_size_sweep}
\end{figure}

\newpage

{\color{black}

\section{Quadratic functions do not have corridors} \label{appendix:quadratic}
In this section, we show that quadratic functions which display corridors are linear.

Assume $E(\theta) = \frac{1}{2} \theta^T A \theta + b \theta + c$. We then have $\nabla_{\theta} E = A \theta + b $ and $\nabla^2_{\theta} E = A$. For $U$ to be a corridor, we need that $\nabla^2_{\theta} E \nabla_{\theta} E = A (A \theta + b) = 0, \forall \theta \in U$. Consider $\lambda_i, u_i$ be the $i$'th largest eigenvalue and eigenvector of $A$.

We have that 
\begin{align}
   A (A \theta + b) =  \sum_i \left(\lambda_i^2 {u_i}^T \theta + \lambda_i u_i^T b\right) u_i = 0, \forall \theta \in U
\end{align}

We then have that $\forall i$ either
\begin{itemize}
    \item $\lambda_i = 0$
    \item $\lambda_i {u_i}^T \theta + u_i^T b = 0$
\end{itemize}

if $\lambda_i = 0 \forall i$, we have that $A = 0$ and $E$ is a linear function.
If $A \ne 0$, we have that 
\begin{align}
    \nabla_{\theta} E = A \theta + b  = \sum_i (\lambda_i \theta^T u_i + b^t u_i) u_i = \sum_{i, \lambda_i = 0}  b^t u_i u_i
\end{align}
which is a constant, and again we have that $E$ is linear in $\theta$ in $U$.
}

\end{document}